\documentclass{article}

\usepackage[preprint]{neurips_2025}

\usepackage[utf8]{inputenc} %
\usepackage[T1]{fontenc}    %
\usepackage{hyperref}       %
\usepackage{url}            %
\usepackage{booktabs}       %
\usepackage{amsfonts}       %
\usepackage{nicefrac}       %
\usepackage{microtype}      %
\usepackage{xcolor}         %

\usepackage[most]{tcolorbox}
\usepackage{amssymb}
\usepackage{amsmath}
\usepackage{graphicx}
\usepackage{subcaption}

\usepackage{tabularx}
\usepackage{bm,bbm}
\usepackage{algorithm}
\usepackage{amsmath}
\usepackage{amsthm}
\usepackage{algorithm}
\usepackage{algorithmic}
\usepackage{tikz}
\usetikzlibrary{spy}
\theoremstyle{plain}
\newtheorem{theorem}{Theorem}[section]
\newtheorem{proposition}[theorem]{Proposition}
\newtheorem{lemma}[theorem]{Lemma}

\theoremstyle{definition}

\theoremstyle{remark}

\usepackage[capitalize,noabbrev]{cleveref}

\usepackage{etoc}
\etocdepthtag.toc{mtchapter}
\etocsettagdepth{mtchapter}{subsection}
\etocsettagdepth{mtappendix}{none}
\newcommand{\bs}{\boldsymbol{s}}

\newcommand{\bI}{\boldsymbol{I}}

\newcommand{\bepsilon}{\boldsymbol{\epsilon}}

\newcommand{\bsigma}{\boldsymbol{\sigma}}

\newcommand{\E}{\mathbb{E}}  % Expectation

\title{VarDiU: A Variational Diffusive Upper Bound for One-Step Diffusion Distillation}
\author{Leyang Wang\thanks{Equal contribution. Correspondence to Leyang Wang <leyang.wang.24@ucl.ac.uk>, Mingtian Zhang <m.zhang@cs.ucl.ac.uk> }\\University College London\\
\And
Mingtian Zhang$^*$\\ University College London\\
\AND
Zijing Ou \\
    Imperial College London\\
  \And 
  David Barber \\ University College London}
\begin{document}
\maketitle
\begin{abstract}
Recently, diffusion distillation methods have compressed thousand-step teacher diffusion models into one-step student generators while preserving sample quality. Most existing approaches train the student model using a diffusive divergence whose gradient is approximated via the student’s score function, learned through denoising score matching (DSM). Since DSM training is imperfect, the resulting gradient estimate is inevitably biased, leading to sub-optimal performance. In this paper, we propose \emph{VarDiU} (pronounced /va:rdju:/), a Variational Diffusive Upper Bound that admits an unbiased gradient estimator and can be directly applied to diffusion distillation. Using this objective, we compare our method with Diff-Instruct and demonstrate that it achieves higher generation quality and enables a more efficient and stable training procedure for one-step diffusion distillation.
\end{abstract}
\section{Introduction}
Diffusion models~\cite{sohl2015deep,ho2020denoising,song2019generative,songscore} have achieved remarkable success in generating high-quality samples. However, sampling typically requires a large number of  denoising steps, making the process computationally inefficient in practice. To address this limitation, many acceleration methods~\cite{ho2020denoising,nichol2021improved,song2020denoising,liu2022pseudo,lu2022dpm,bao2022analytic,bao2022estimating,ou2024improving,debortoli2025distributionaldiffusionmodelsscoring,xiao2021tackling,yu2024hierarchical} have been proposed to reduce the number of function evaluations (NFEs) in pre-trained diffusion models, lowering the cost from thousands of steps to only tens. More recently, distillation-based approaches have further compressed the sampling process to as few as a single step while largely preserving generation quality~\cite{zhou2024score,salimansprogressive,berthelot2023tract,song2023consistency,heek2024multistep,kim2023consistency,li2024bidirectional,luo2023diff,salimans2024multistep,xie2024distillation,zhang2025towards}.

In this paper, we investigate a family of distillation methods based on divergence minimisation~\cite{luo2023diff,zhou2024adversarial,zhang2025towards,zhou2024score}. These methods compress the full denoising process of a pre-trained teacher diffusion model into a one-step latent variable model by minimising a diffusive divergence (e.g. Diffusive Reverse KL Divergence~\cite{he2024training,zhang2025towards,luo2023diff}) between the student and teacher through their respective score estimations. Although these approaches have demonstrated promising results, their gradients are approximated using the student’s score function, which is learned through denoising score matching (DSM). Because DSM training is inherently imperfect, the resulting gradient estimates are inevitably biased, often leading to sub-optimal generation performance.

In this study, we introduce \emph{Variational Diffusive Upper Bound Distillation (VarDiU)}, a method that formulates a variational upper bound of the diffusive divergence whose gradient can be estimated \textbf{unbiasedly}. We demonstrate that this unbiased gradient estimation leads to more stable and efficient training, achieving better performance compared to the original diffusive divergence, which relies on biased gradient approximations.

\section{Background on Diffusion Models}

Let $p_d(x_0)$ denote the data distribution. Diffusion models define a forward noising process as
$q(x_{0:T}) = p_d(x_0)\prod_{t=1}^T q(x_t | x_{t-1})$
where the transition kernel is a fixed Gaussian distribution
$
q(x_t | x_{t-1}) = \mathcal{N}(x_t |\sqrt{1 - \beta_t}x_{t-1},\beta_t I).
$
This process admits the closed-form distribution,
\begin{equation}
    q(x_t | x_0) = \mathcal{N}(x_t | \sqrt{\bar{\alpha}_t} x_0, (1 - \bar{\alpha}_t) I),
\end{equation}
with $\bar{\alpha}_t = \prod_{s=1}^t (1 - \beta_s)$. As $T \to \infty$, the final state $x_T$ converges to a standard Gaussian distribution, i.e., $q(x_T) \to \mathcal{N}(0, I)$. The generative process also uses a sequence of Gaussian distributions:
\begin{equation}
    p(x_{t-1} | x_t) = \mathcal{N}\!\left(x_{t-1}|\mu_{t-1}(x_t), \Sigma_{t-1}(x_t)\right),
\end{equation}
where the covariance $\Sigma_{t-1}(x_t)$ may either be learned~\cite{nichol2021improved,ou2024improving,bao2022estimating} or fixed~\cite{bao2022analytic,ho2020denoising}.  
The mean is estimated via Tweedie’s formula~\cite{efron2011tweedie,robbins1992empirical}:
\begin{equation}
    \mu_{t-1}(x_t) = \frac{1}{\sqrt{1 - \beta_t}} \Bigl(x_t + \beta_t \nabla_{x_t} \log p_d(x_t)\Bigr),
\end{equation}
where the score function $\nabla_{x_t}\log p_d(x_t)$ is approximated with a network $\bs_{\psi}(x_t;t)$ which is learned through denoising score matching (DSM)~\cite{vincent2011connection,song2019generative}:
\begin{equation}
    \mathcal{L}_{\mathrm{DSM}}(\psi) = \int p_d(x_0)\, q(x_t | x_0)\, \bigl\| \bs_{\psi}(x_t;t) - \nabla_{x_t}\log q(x_t | x_0)\bigr\|_2^2 \, \mathrm{d}x_t\, \mathrm{d}x_0.
\end{equation}
The classical diffusion models require thousands of time steps to generate high-quality samples, making them computationally expensive. In the following section, we discuss a family of methods that can distill multi-step diffusion processes into one-step generative models.

\subsection{Score-Based Distillation via Diffusive Reverse KL Minimisation}
\label{sec:dikl}
Score-based distillation methods aim to compress a teacher diffusion model, which we also denote as $p_d$, into a one-step implicit generative model~\cite{goodfellow2014generative,huszar2017variational,zhang2020spread}, defined as
\begin{align}
    p_\theta(x_0) = \int \delta(x_0 - g_\theta(z))\, p(z)\, dz,
\end{align}
where $\delta(\cdot)$ is the Dirac delta function, $p(z)$ is a standard Gaussian prior, and $g_\theta$ is a deterministic neural network that generates samples in a single step. The training objective is formulated as the \emph{Diffusive KL  divergence} (DiKL)~\cite{zhang2020spread,wang2024prolificdreamer,luo2023diff,he2024training,zhang2025towards}, which measures the discrepancy between the student and teacher distributions across all intermediate noise levels:
\begin{align}
\label{eq:dikl}
   \mathrm{DiKL}(p_{\theta}||p_d) \equiv \int_{0}^1 \omega(t) \, \mathrm{KL}(p^{(t)}_{\theta}(x_t) \| p_d^{(t)}(x_t)) \, \mathrm{d}t,
\end{align}
where $\omega(t)$ is a positive weighting function. Here, $p_\theta^{(t)}(x_t) = \int q(x_t | x_0)\, p_\theta(x_0)\, dx_0,$
and $p_d^{(t)}(x_t)$ is the marginal distribution induced by the pre-trained teacher diffusion model under the same noise kernel $q(x_t | x_0)$. This divergence is well defined even when the underlying distributions have disjoint support or don't even allow density functions~\cite{zhang2020spread}. However, direct estimation of this divergence is intractable. One can derive the gradient w.r.t. ${\theta}$ at time $t$ as
\begin{align}
\label{eq:diklgrad}
   \nabla_{\theta}   \mathrm{KL}(p^{(t)}_{\theta} \| p_d^{(t)}) 
    = \int
    \left[
        \left(\nabla_{x_t}\log p^{(t)}_{\theta}(x_t)
        - \nabla_{x_t}\log p^{(t)}_{d}(x_t)\right)
        \frac{\partial x_t}{\partial {\theta}}
    \right] \mathrm{d}x_t,
\end{align}
see A.2 in \citep{luo2023diff} for a derivation.
The teacher score $\nabla_{x_t}\log p^{(t)}_{d}(x_t)$ is available from the pre-trained diffusion model. However, the student score $\nabla_{x_t}\log p^{(t)}_{\theta}(x_t)$ is unknown. Fortunately, since we can sample from the student distribution, it can be estimated using denoising score matching (DSM):
\begin{equation}
    \mathcal{L}_{\mathrm{DSM}}(\psi^\prime) = \int p_\theta(x_0)\, q(x_t | x_0)\, \bigl\| \bs_{\psi^\prime}(x_t;t) - \nabla_{x_t}\log q(x_t | x_0)\bigr\|_2^2 \, \mathrm{d}x_t\, \mathrm{d}x_0.
\end{equation}
This method is known as VSD~\cite{wang2024prolificdreamer} or Diff-Instruct~\cite{luo2023diff} in the context of diffusion distillation.
In practice, DSM approximations are imperfect due to the limited capacity of the score network $\bs_{\psi^\prime}$. As a result, \emph{the direct gradient estimation of DiKL is biased}, which leads to suboptimal training of the distilled student model. In the next section, we introduce a variational upper bound on the reverse KL divergence that enables unbiased gradient estimation for diffusion distillation.

\section{Variational Diffusive Upper Bound Distillation}
Instead of relying on optimisation-based estimation of the DiKL objective, we propose a variational upper bound whose gradient can be estimated unbiasedly. Specifically, we propose to extend the RKL upper bound of \citep{zhang2019variational} into the diffusion space,
\begin{align*}
    \mathrm{KL}(p^{(t)}_{\theta}(x)||p^{(t)}_d(x)) &\leq \mathrm{KL}(p^{(t)}_{\theta}(x_t|z)p(z)||p^{(t)}_d(x)q^{(t)}_{\phi}(z|x_t)) \equiv \mathrm{U}^{(t)}(\theta,\phi).
\end{align*}
The inequality follows by Jensen’s inequality; and the equality can be attained when the variational posterior is equal to the true posterior $q_\phi(z|x_t;t)=p_\theta(z|x_t)\propto p_\theta(x_t|z)p(z)$ (see~\cref{app:bound_proof} for a proof). 
We are then ready to define the diffusive upper bound as 
\begin{align}
    \label{eq:diu}
    \mathrm{DiU}(\theta,\phi) \equiv \int_0^1 w(t) \mathrm{U}^{(t)}(\theta,\phi) \mathrm{d}t \geq 
    \int_0^1 w(t)\mathrm{KL}(p^{(t)}_{\theta}||p^{(t)}_d) \mathrm{d}t \equiv \mathrm{DiKL}(p_{\theta}||p_d),
\end{align}
where the bound is tight when $q^{(t)}_\phi(z|x_t)=p^{(t)}_\theta(z|x_t)\propto p^{(t)}_\theta(x_t|z)p(z)$ for all $t$ almost everywhere.

\subsection{Tractable Upper Bound Estimation}
We then discuss how to estimate this upper bound. Writing the model joint as $p^{(t)}_{\theta}(x_t,z)=p^{(t)}_{\theta}(x_t| z)p(z)$, the upper bound at time step $t$ can be written as
\begin{align}
    \label{eq:diuexpand}
    \mathrm{U}^{(t)}(\theta,\phi) = -\mathbb{H}(p^{(t)}_{\theta}(x_t,z))
    - \iint p^{(t)}_{\theta}(x_t,z) \left(\log p^{(t)}_d(x_t)\mathrm{d}x_t
    -  \log q^{(t)}_{\phi}(z|x_t)\right)\,\mathrm{d}z \mathrm{d}x_t.
\end{align}
where $\mathbb{H}(p^{(t)}_{\theta}(x_t,z))$ denotes the joint entropy, i.e. $\mathbb{H}(p)=-\int p(x)\log p(x)\mathrm{d}x$.  In the original reverse KL upper bound in~\citep{zhang2019variational}, the joint entropy is trained with the reparametrisation trick. In our case, when the student model is an implicit model, we present the following surprising fact that \emph{the joint entropy is a constant}, which removes the requirement of joint entropy estimation.
\begin{proposition}
    \label{prop:joint}
    Let $p_{\theta}(x,z)=p_{\theta}(x |z)p(z)$
    with a fixed prior \(p(z)\) and 
    $p_{\theta}(x|z)=\mathcal{N}\bigl(x;\,g_{\theta}(z),\sigma^2{I}\bigr)$.  Then the joint entropy
    $\mathbb{H}(p_{\theta}(x,z))$ is a constant and independent of $\theta$.
\end{proposition}
See \cref{sec:entropy_proof} for a proof. This should be distinguished from the marginal entropy $\mathbb{H}(p^{(t)}_{\theta}(x_t))$, which still depends on $\theta$, also see Section~\ref{sec:im} for a deeper discussion.  This observation greatly simplifies training by eliminating the need for entropy estimation.

The second term of Equation~\ref{eq:diu} cannot be directly optimised via automatic differentiation because the teacher density $p^{(t)}_d(x_t)$ is unknown. However, in typical distillation settings, the score function $\nabla_{x_t}\log p^{(t)}_d(x_t)$ is assumed to be available. Leveraging this, we propose an alternative loss whose gradient is equivalent and can be computed via the score function:
\begin{align}
    \nabla_{{\theta}} \int p^{(t)}_{\theta}(x_t) \log p^{(t)}_d(x_t) \mathrm{d}x_t= \nabla_{{\theta}} \int p^{(t)}_\theta(x_t) \left(x_t^\top [\nabla_{x_t}\log p^{(t)}_d(x_t)]_{\text{sg}}\right) \mathrm{d}x_t,
    \label{eq:grad}
\end{align}
where $[\,\cdot\,]_{\mathrm{sg}}$ denotes the stop-gradient operator.
The proof involves using the reparametrisation trick and chain rule; see \cref{app:proofgrad} for details. The final loss objective at time $t$ becomes
\begin{align}
    \label{eq:diuexpand}
    \mathrm{U}^{(t)}(\theta,\phi) \doteq 
    - \iint p^{(t)}_{\theta}(x_t,z) \left( x_t^\top [\nabla_{x_t}\log p^{(t)}_d(x_t)]_{\text{sg}}
    + \log q^{(t)}_{\phi}(z|x_t)\right)\,\mathrm{d}z \mathrm{d}x_t,
\end{align}
where $\doteq$ denotes equivalence up to a constant. Therefore, when using a simple Gaussian variational family \( q^{(t)}_\phi(z|x_t) \), this bound can be unbiasedly estimated. A loss used practically can be found in \cref{eq:eqloss}. We also propose a novel noise schedule and variance reduction techniques for training with this objective, see \cref{app:schedule} for details.

\subsection{Improving Variational Inference with Normalising Flows} 
The upper bound becomes tight, i.e., equal to the DiKL objective, when the variational posterior matches the true posterior: \( q^{(t)}_\phi(z|x_t) = p^{(t)}_\theta(z|x_t) \). Simple Gaussian variational distributions often struggle to approximate complex posteriors, particularly in cases where the true posterior is multi-modal or exhibits strong correlations. To address this limitation, we could use normalising flow ~\cite{tabak2010density,tabak2013family,dinh2014nice,rezende2015variational,zhai2024normalizing,durkan2019neural,kobyzev2020normalizing} to increase the flexibility of the posterior ~\cite{rezende2015variational}. Specifically, we define:
\begin{align}
q^{(t)}_{\phi}(z | x_t) = \int \delta\left(z - f_{\phi}(a,x_t, t)\right) r_{\phi}(a | x_t; t) \,\mathrm{d}a,
\end{align}
where $f$ is an invertible function and the base distribution is defined as a Gaussian indexed by $t$.
\begin{align}
r_{\phi}(a | x_t; t) = \mathcal{N}\left(a; \mu_{\phi}(x_t; t),\, \sigma_t^2 \mathrm{diag}(\sigma^2_{\phi}(x_t; t))\right).
\end{align}
The time-dependent scaling factor \( \sigma_t \) modulates the entropy of the base distribution to match the noise level at diffusion step \( t \). 
The corresponding log-density of \( q^{(t)}_\phi(z | x_t) \) can be computed using the change-of-variables formula:
\begin{align}
\log q^{(t)}_{\phi}(z = f_{\phi}(a; t) | x_t; t)
= \log r_{\phi}(a | x_t; t) - \log \left| \det \frac{\partial f_{\phi}}{\partial a} \right|,
\end{align}
We can then substitute this log-density into Equation~\eqref{eq:diuexpand} to obtain a tractable objective for distillation.

\subsection{An Information Maximisation View of the Variational Diffusive Upper Bound}\label{sec:im}

In the previous section, we showed that the joint entropy term in Equation~\eqref{eq:diuexpand} is constant with respect to \(\theta\). As a result, the variational upper bound at diffusion step \(t\) simplifies to:
\begin{align}
    \label{eq:diuexpand_2}
    \mathrm{U}^{(t)}(\theta,\phi) =
    - \int p^{(t)}_{\theta}(x_t) \log p^{(t)}_d(x_t)\,\mathrm{d}x_t
    - \iint p^{(t)}_{\theta}(x_t,z) \log q^{(t)}_{\phi}(z|x_t)\,\mathrm{d}z\,\mathrm{d}x_t.
\end{align}
Now consider the original DiKL divergence at time step \(t\), defined as:
\begin{align}
    \label{eq:original:KL}
    \mathrm{KL}(p^{(t)}_{\theta} \,\|\, p^{(t)}_d) 
    = - \int p^{(t)}_{\theta}(x_t)\log p^{(t)}_d(x_t)\,\mathrm{d}x_t
    + \int p^{(t)}_{\theta}(x_t) \log p^{(t)}_{\theta}(x_t)\,\mathrm{d}x_t.
\end{align}
We observe that the first term in both objectives (Equations~\eqref{eq:diuexpand_2} and \eqref{eq:original:KL}) is identical. The second term in Equation~\eqref{eq:original:KL} corresponds to the negative entropy of the marginal distribution, i.e., \( -\mathbb{H}(p^{(t)}_\theta(x_t)) \), which depends on \(\theta\).
We can then obtain an alternative perspective for justifying the proposed upper bound, formalised in the following theorem:
\begin{theorem}
\label{thm:entropy}
By the chain rule of entropy, we have:
$\mathbb{H}(p^{(t)}_\theta(x_t, z)) = \mathbb{H}(p^{(t)}_\theta(z | x_t)) + \mathbb{H}(p^{(t)}_\theta(x_t))$.
Since \( \mathbb{H}(p^{(t)}_\theta(x_t, z)) \) is constant with respect to \(\theta\), it follows that:
\begin{align}
\max_{\theta} \mathbb{H}(p^{(t)}_\theta(x_t))
= \min_{\theta} \mathbb{H}(p^{(t)}_\theta(z|x_t))
\le \min_{\theta, \phi} \;
- \iint p^{(t)}_{\theta}(x_t, z) \log q^{(t)}_{\phi}(z | x_t) \,\mathrm{d}z\,\mathrm{d}x_t,
\end{align}
where \( q^{(t)}_{\phi}(z|x_t) \) is a variational approximation to the true posterior \( p^{(t)}_\theta(z|x_t) \).
\end{theorem}
See \cref{app:entropy} for a proof. This result mirrors the variational approach in the Information Maximisation (IM) algorithm \citep{barber2004algorithm}, which maximises mutual information by minimising the conditional entropy using a variational decoder. Hence, this entropy bound can be viewed as a continuous latent-variable adaptation of the IM algorithm, where entropy maximisation serves to promote informative and diverse model outputs. 

\section{Experiments}
To demonstrate the effectiveness of our approach, we consider a toy 2d mixture of 40 Gaussians (MoG-40) proposed in \cite{midgleyflow}, where the means are uniformly distributed over $[-40, 40]^2$. We train VarDiU with two types of posteriors: \textbf{(a)} a Gaussian posterior with learnable mean and variance (VarDiU-Gaussian), and \textbf{(b)} a Neural Spline Flow (NSF; \cite{durkan2019neural}) with a learnable-parameters Gaussian base distribution (VarDiU-NSF). We compare VarDiU to Diff-Instruct \citep{luo2023diff} using varying numbers of student score training steps $(1,5,10)$. 
Experimental details can be found in \cref{app:exp_details}.

\textbf{Settings} We consider three practical settings:  
\textbf{(1)} access to the true analytical score;  
\textbf{(2)} access to the training dataset; %
\textbf{(3)} access to a teacher score pre-trained by a diffusion model. For the given training data case, we use the  empirical score, defined by $\nabla_{x_t}\log \hat{p}_d(x_t)=\frac{1}{N}\sum_{n=1}^N \nabla_{x_t}\log q(x_t|x_0^{(n)})$ with the given dataset $\mathcal{D}=\{x_0^{(n)}\}^{N}_{n=1}$, where $q(x_t|x_0)$ matches the diffusion schedule of the VarDiU. This empirical score is a consistent estimator of the true data score~\cite{song2019generative,song2020score,zhang2019variational}, see \cref{eq:emp_score} for details. The learned score estimator is obtained by training EDM \citep{karras2022edm}, which is a popular class of diffusion models, on the observed data. For each setting, we report results over 10 independent runs. The noise schedule and weighting strategy are detailed in \cref{app:schedule}.

\textbf{Metrics} We report the \emph{log-density} of generated samples under the true data distribution, which reflects sample quality but not sample diversity. We additionally compute the Maximum Mean Discrepancy (MMD; \cite{gretton2012kernel}) with five kernels which can measure diversity, see for details in \cref{app:imp_detail}.

\begin{table}
  \centering
  \caption{\textbf{Comparison across different settings.} ($\uparrow$) denotes the higher is better and ($\downarrow$) denotes the lower is better. Best performance is in bold. Second best is underlined. \textit{“True”} indicates the samples from the target distribution. Diff-inst ($k$) means Diff-instruct trained with $k$ score steps. We can find VarDiU achieve best results with true and empirical score and remain competitive with learned score.}
  \label{tab:mog40_all}
  
  \begin{subtable}[t]{\linewidth}
    \centering
    \subcaption{Comparison with true score}
    \label{tab:truemog40}
    \resizebox{\linewidth}{!}{
    \begin{tabular}{lrrrrrr}
      \toprule
      \textbf{Metrics} & \textit{True} & Diff-Inst (1) & Diff-Inst (5) & Diff-Inst (10) & VarDiU Gaussian & VarDiU NSF\\
      \midrule
      Log-Density ($\uparrow$) & \textit{-6.65} & $-10.17 \pm 1.77$ & $-8.09 \pm 0.86$ & $-7.86 \pm 0.56$ & $\underline{-7.00 \pm 0.02}$ & $\boldsymbol{-6.99 \pm 0.01}$ \\
      Log-MMD ($\downarrow$) & \textit{/} & $-5.59 \pm 0.36$ & $-6.09 \pm 0.40$ & $-5.79 \pm 0.36$ & $\underline{-6.86 \pm 0.29}$ & $\boldsymbol{-7.33 \pm 0.17}$\\
      \bottomrule
    \end{tabular}}
  \end{subtable}
  \hfill
  \begin{subtable}[t]{\linewidth}
    \centering
    \subcaption{Comparison with training data }
    \label{tab:kdemog40}
    \resizebox{\linewidth}{!}{
    \begin{tabular}{lrrrrrrr}
      \toprule
      \textbf{Metrics} & \textit{True} &  Diff-Inst (1) & Diff-Inst (5) & Diff-Inst (10) & VarDiU Gaussian & VarDiU NSF\\
      \midrule
      Log-Density ($\uparrow$) & \textit{-6.65}  & $-9.65 \pm 0.62$ & $-8.38 \pm 0.42$ & $-7.96 \pm 0.62$ & $\underline{-7.09 \pm 0.03}$ &  $\boldsymbol{-7.01 \pm 0.02}$ \\
      Log-MMD ($\downarrow$) & \textit{/}  & $-5.51 \pm 0.26$ & $-5.60 \pm 0.30$ & $-5.63 \pm 0.68$ & $\underline{-6.36 \pm 0.22}$ & $\boldsymbol{-6.81 \pm 0.30}$\\
      \bottomrule
    \end{tabular}}
  \end{subtable}
  \hfill
  \begin{subtable}[t]{\linewidth}
    \centering
    \subcaption{Comparison with learned score}
    \label{tab:edmmog40}
    \resizebox{\linewidth}{!}{
    \begin{tabular}{lrrrrrrr}
      \toprule
      \textbf{Metrics} & \textit{True} & EDM & Diff-Inst (1) & Diff-Inst (5) & Diff-Inst (10) & VarDiU Gaussian & VarDiU NSF\\
      \midrule
      Log-Density ($\uparrow$) & \textit{-6.65} & -6.69 & $-10.88 \pm 1.55$ & $-9.21 \pm 0.70$ & $-8.47 \pm 0.56$ & $\underline{-8.13 \pm 0.19}$ & $\boldsymbol{-7.89 \pm 0.15}$ \\
      Log-MMD ($\downarrow$) & \textit{/} & -6.67 & $-5.24 \pm 0.30$ & $-5.57 \pm 0.27$ & $\boldsymbol{-5.82 \pm 0.27}$ & $-5.64 \pm 0.25$ & $\underline{-5.68 \pm 0.30}$\\
      \bottomrule
    \end{tabular}}
  \end{subtable}
\end{table}
\begin{figure}
    \centering
    \begin{subfigure}[b]{0.24\textwidth}
        \centering
        \includegraphics[width=\linewidth]{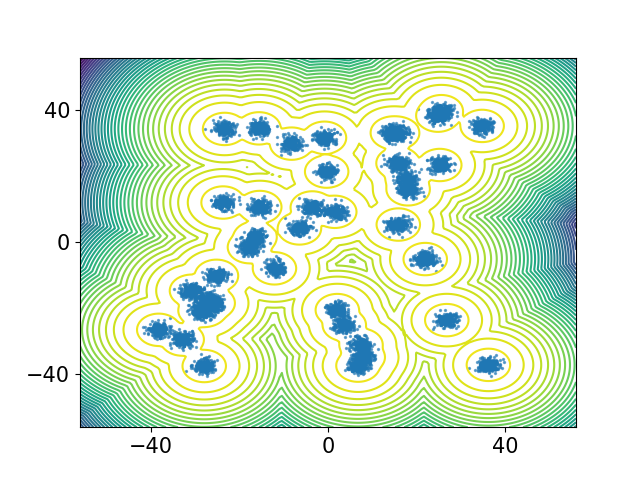}
        \caption{True Samples}
    \end{subfigure}
    \begin{subfigure}[b]{0.24\textwidth}
        \centering
        \includegraphics[width=\linewidth]{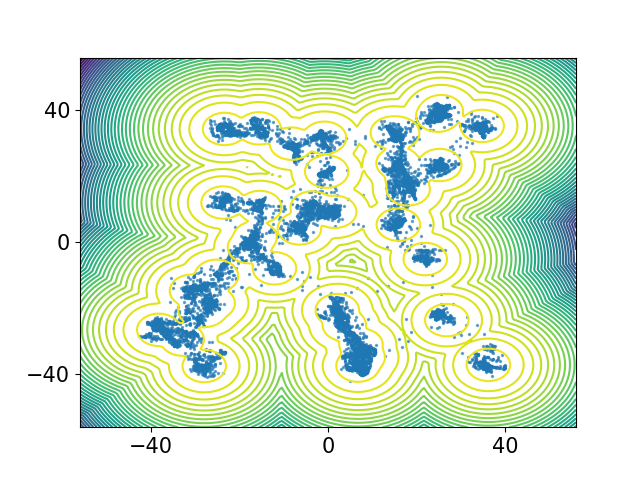}
        \caption{Diff-Instruct}
    \end{subfigure}
    \begin{subfigure}[b]{0.24\textwidth}
        \centering
        \includegraphics[width=\linewidth]{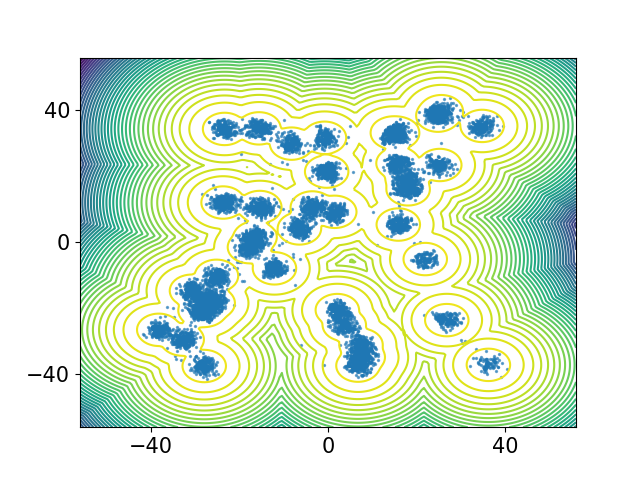}
        \caption{VarDiU Gaussian}
    \end{subfigure}
    \begin{subfigure}[b]{0.24\textwidth}
        \centering
        \includegraphics[width=\linewidth]{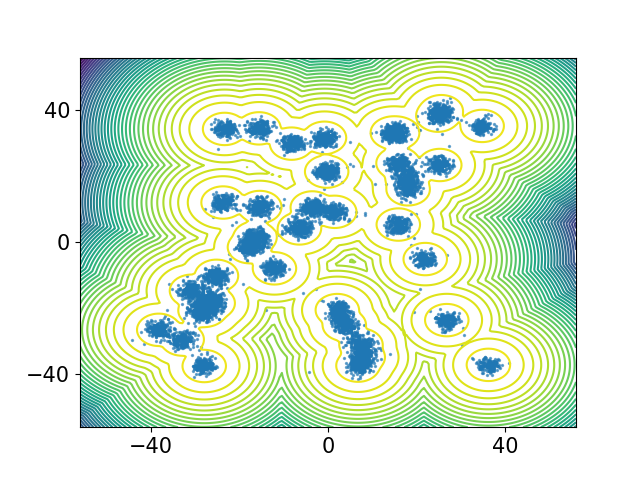}
        \caption{VarDiU NSF}
    \end{subfigure}
    \begin{subfigure}[b]{0.24\textwidth}
        \centering
        \includegraphics[width=\linewidth]{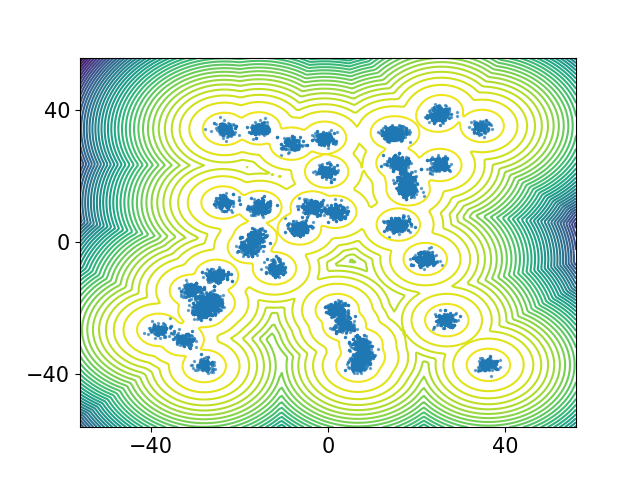}
        \caption{Empirical EDM}
    \end{subfigure}
     \begin{subfigure}[b]{0.24\textwidth}
        \centering
        \includegraphics[width=\linewidth]{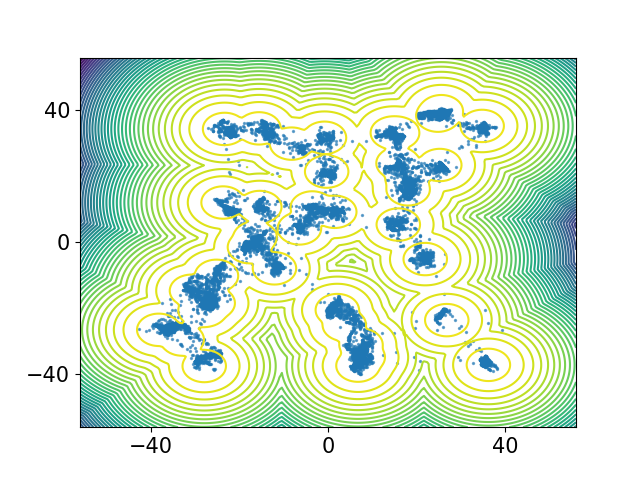}
        \caption{Diff-Instruct}
    \end{subfigure}
     \begin{subfigure}[b]{0.24\textwidth}
        \centering
        \includegraphics[width=\linewidth]{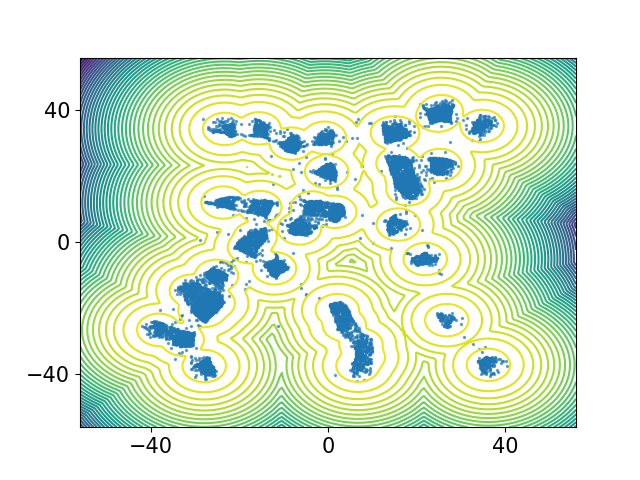}
        \caption{VarDiU Gaussian}
    \end{subfigure}
    \begin{subfigure}[b]{0.24\textwidth}
        \centering
        \includegraphics[width=\linewidth]{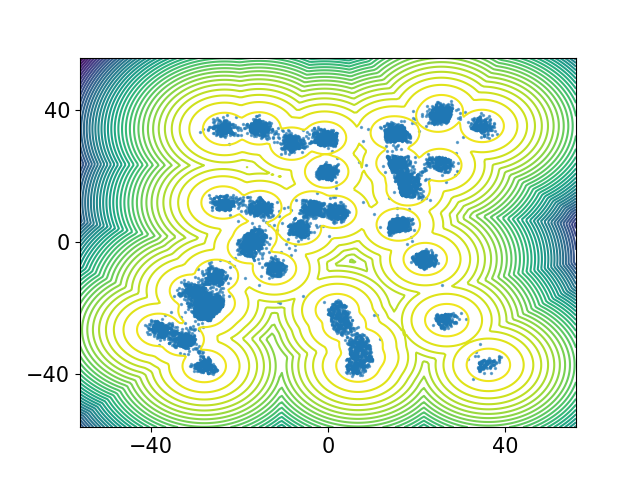}
        \caption{VarDiU NSF}
    \end{subfigure}
    \begin{subfigure}[b]{0.24\textwidth}
        \centering
        \includegraphics[width=\linewidth]{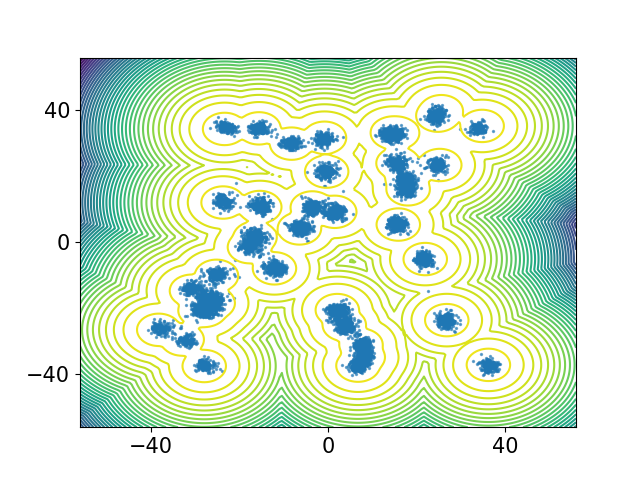}
        \caption{Pre-trained EDM}
    \end{subfigure}
    \begin{subfigure}[b]{0.24\textwidth}
        \centering
        \includegraphics[width=\linewidth]{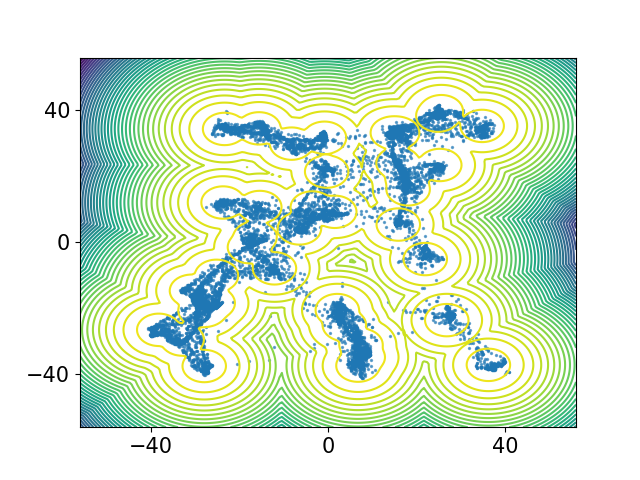}
        \caption{Diff-Instruct}
    \end{subfigure}
    \begin{subfigure}[b]{0.24\textwidth}
        \centering
        \includegraphics[width=\linewidth]{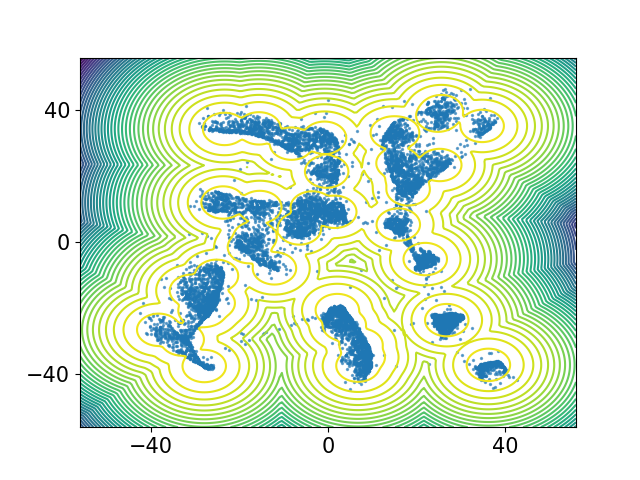}
        \caption{VarDiU Gaussian}
    \end{subfigure}
    \begin{subfigure}[b]{0.24\textwidth}
        \centering
        \includegraphics[width=\linewidth]{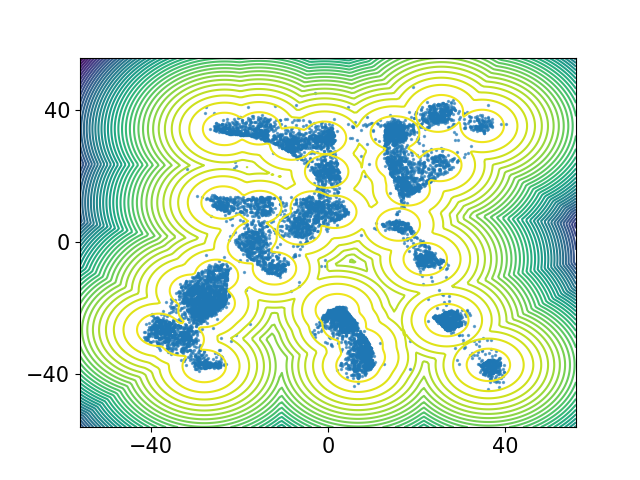}
        \caption{VarDiU NSF}
    \end{subfigure}
    \caption{The three rows correspond to the true score, training data, and learned score settings. Diff-Instruct models use 10 score-matching steps per generator step. Empirical EDM indicates the samples generated by EDM sampler with empirical scores. See~\cref{tab:truemog40,tab:edmmog40,tab:kdemog40} for evaluation results.}
    \label{fig:mog40samples}
\end{figure}
\begin{figure}[t]
    \centering
        \centering
\includegraphics[width=0.92\linewidth]{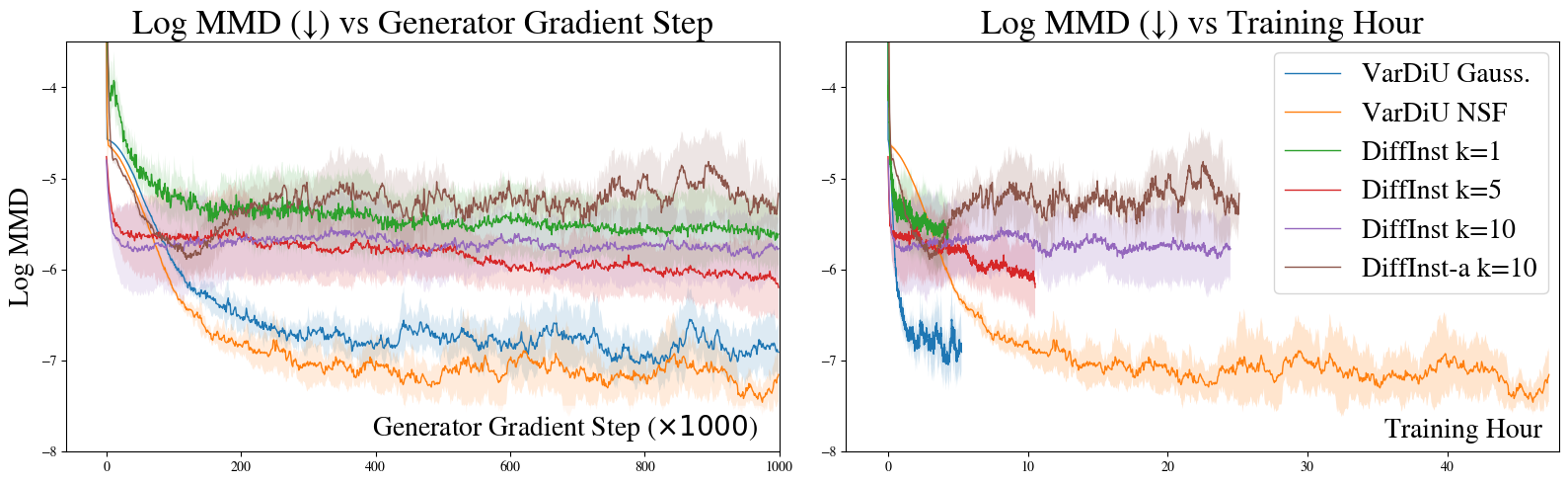}
        \caption{The left figure shows log-MMD over 1M generator steps, reflecting the efficiency of generator gradient estimation across different methods. The right figure shows total training time for 1M generator steps. In both, VarDiU clearly surpasses Diff-Instruct in quality, stability, and efficiency.}
\label{fig:speed_true}
\end{figure}

\textbf{Sample Quality}
We compare samples generated by one-step models trained with VarDiU and Diff-Instruct across three practical settings: using the true score, the training dataset, and a learned score (see \cref{fig:mog40samples} for visualizations). For Diff-Instruct, we train the model with 10 score-matching steps per generator gradient step to obtain the best sample quality we can, whereas the original Diff-Instruct paper~\cite{luo2023diff} uses only a single step. Both VarDiU-Gaussian and VarDiU-NSF produce cleaner and sharper samples than Diff-Instruct when using the true score or the training dataset. With a learned score, however, all methods yield different sub-optimal results due to biases in the provided scores. We also compare log-density and log-MMD across the samples (see \cref{tab:truemog40,tab:edmmog40,tab:kdemog40}). We can find VarDiU consistently performs best under reliable evaluations (true score and training data), demonstrating high accuracy and stability with tight confidence intervals, while Diff-Instruct lags despite larger steps. Under learned scores, evaluations are noisier: VarDiU achieves higher log-density but slightly worse MMD, making comparisons unreliable due to bias in the given pre-trained scores. Overall, VarDiU is more accurate and stable across settings.

\textbf{Training Stability and Efficiency}
We present plots of log-MMD against both generator gradient iterations and total training time in \cref{fig:speed_true}; additional results are provided in Appendix \cref{fig:true_scorelogpd,fig:speed_fake,fig:speed_emp}. Each curve shows the mean of 10 independent runs, with shaded regions indicating one standard deviation. For VarDiU, we apply a noise scheme with an annealing schedule. For Diff-Instruct, we use both the same annealing schedule (denoted as Diff-Instruct-a) and a non-annealing variant (see \cref{app:schedule} for details). In the log-MMD vs. generator gradient step plot, we find that under the same generator gradient budget (1M steps), VarDiU methods significantly outperform Diff-Instruct in both MMD values and stability (smaller variance). Moreover, incorporating a normalizing flow further improves performance compared to the Gaussian posterior, highlighting the importance of accurate generator gradient estimation. In the log-MMD vs. training time plot, we observe that under the same training hours, VarDiU-Gaussian clearly outperforms all other methods.

\section{Conclusion and Future Work}
In this paper, we introduced a novel method for diffusion distillation by proposing a variational diffusive upper bound on the DiKL divergence whose gradient can be estimated without bias. Our bound extends the Reverse KL upper bound from~\citep{zhang2019variational} to the diffusion setting and further eliminates the need to estimate the joint entropy. We also adapt this bound naturally to scenarios where only the teacher model's score function is accessible. Additionally, we establish a connection between our method and the information maximisation (IM~\cite{barber2004algorithm}) framework, see Appendix~\ref{app:related} for a discussion on related works. Empirically, we demonstrate on a toy problem that our method achieves improved one-step generation quality while significantly reducing training time. 

As future work, we are extending this approach to higher-dimensional data, such as image or video generation, which we believe to be a promising direction.

\clearpage
\newpage

{
\small
\bibliographystyle{abbrv}
\bibliography{SPIGM}

\begin{thebibliography}{10}

\bibitem{bao2022estimating}
F.~Bao, C.~Li, J.~Sun, J.~Zhu, and B.~Zhang.
\newblock Estimating the optimal covariance with imperfect mean in diffusion probabilistic models.
\newblock {\em arXiv preprint arXiv:2206.07309}, 2022.

\bibitem{bao2022analytic}
F.~Bao, C.~Li, J.~Zhu, and B.~Zhang.
\newblock Analytic-dpm: an analytic estimate of the optimal reverse variance in diffusion probabilistic models.
\newblock {\em arXiv preprint arXiv:2201.06503}, 2022.

\bibitem{barber2004algorithm}
D.~Barber and F.~Agakov.
\newblock The im algorithm: a variational approach to information maximization.
\newblock {\em Advances in neural information processing systems}, 2004.

\bibitem{berthelot2023tract}
D.~Berthelot, A.~Autef, J.~Lin, D.~A. Yap, S.~Zhai, S.~Hu, D.~Zheng, W.~Talbott, and E.~Gu.
\newblock Tract: Denoising diffusion models with transitive closure time-distillation.
\newblock {\em arXiv preprint arXiv:2303.04248}, 2023.

\bibitem{debortoli2025distributionaldiffusionmodelsscoring}
V.~D. Bortoli, A.~Galashov, J.~S. Guntupalli, G.~Zhou, K.~Murphy, A.~Gretton, and A.~Doucet.
\newblock Distributional diffusion models with scoring rules, 2025.

\bibitem{botev2017variance}
Z.~Botev and A.~Ridder.
\newblock Variance reduction.
\newblock {\em Wiley statsRef: Statistics reference online}, 136:476, 2017.

\bibitem{de2021diffusion}
V.~De~Bortoli, J.~Thornton, J.~Heng, and A.~Doucet.
\newblock Diffusion schr{\"o}dinger bridge with applications to score-based generative modeling.
\newblock {\em Advances in Neural Information Processing Systems}, 2021.

\bibitem{dinh2014nice}
L.~Dinh, D.~Krueger, and Y.~Bengio.
\newblock Nice: Non-linear independent components estimation.
\newblock {\em arXiv preprint arXiv:1410.8516}, 2014.

\bibitem{durkan2019neural}
C.~Durkan, A.~Bekasov, I.~Murray, and G.~Papamakarios.
\newblock Neural spline flows.
\newblock {\em Advances in Neural Information Processing Systems}, 2019.

\bibitem{efron2011tweedie}
B.~Efron.
\newblock Tweedie’s formula and selection bias.
\newblock {\em Journal of the American Statistical Association}, 106(496):1602--1614, 2011.

\bibitem{goodfellow2014generative}
I.~Goodfellow, J.~Pouget-Abadie, M.~Mirza, B.~Xu, D.~Warde-Farley, S.~Ozair, A.~Courville, and Y.~Bengio.
\newblock Generative adversarial nets.
\newblock {\em Advances in neural information processing systems}, 2014.

\bibitem{gretton2012kernel}
A.~Gretton, K.~M. Borgwardt, M.~J. Rasch, B.~Sch{\"o}lkopf, and A.~Smola.
\newblock A kernel two-sample test.
\newblock {\em The journal of machine learning research}, (1):723--773, 2012.

\bibitem{he2024training}
J.~He, W.~Chen, M.~Zhang, D.~Barber, and J.~M. Hern{\'a}ndez-Lobato.
\newblock Training neural samplers with reverse diffusive kl divergence.
\newblock In {\em International Conference on Artificial Intelligence and Statistics}. PMLR, 2025.

\bibitem{heek2024multistep}
J.~Heek, E.~Hoogeboom, and T.~Salimans.
\newblock Multistep consistency models.
\newblock {\em arXiv preprint arXiv:2403.06807}, 2024.

\bibitem{ho2020denoising}
J.~Ho, A.~Jain, and P.~Abbeel.
\newblock Denoising diffusion probabilistic models.
\newblock {\em Advances in neural information processing systems}, 2020.

\bibitem{huszar2017variational}
F.~Husz{\'a}r.
\newblock Variational inference using implicit distributions.
\newblock {\em arXiv preprint arXiv:1702.08235}, 2017.

\bibitem{karras2022edm}
T.~Karras, M.~Aittala, T.~Aila, and S.~Laine.
\newblock Elucidating the design space of diffusion-based generative models.
\newblock {\em Advances in Neural Information Processing Systems}, 2022.

\bibitem{kim2023consistency}
D.~Kim, C.-H. Lai, W.-H. Liao, N.~Murata, Y.~Takida, T.~Uesaka, Y.~He, Y.~Mitsufuji, and S.~Ermon.
\newblock Consistency trajectory models: Learning probability flow {ODE} trajectory of diffusion.
\newblock In {\em International Conference on Learning Representations}, 2024.

\bibitem{kobyzev2020normalizing}
I.~Kobyzev, S.~J. Prince, and M.~A. Brubaker.
\newblock Normalizing flows: An introduction and review of current methods.
\newblock {\em IEEE transactions on pattern analysis and machine intelligence}, 43(11):3964--3979, 2020.

\bibitem{kroese2013handbook}
D.~P. Kroese, T.~Taimre, and Z.~I. Botev.
\newblock {\em Handbook of monte carlo methods}.
\newblock John Wiley \& Sons, 2013.

\bibitem{li2024bidirectional}
L.~Li and J.~He.
\newblock Bidirectional consistency models.
\newblock {\em arXiv preprint arXiv:2403.18035}, 2024.

\bibitem{lipmanflow}
Y.~Lipman, R.~T. Chen, H.~Ben-Hamu, M.~Nickel, and M.~Le.
\newblock Flow matching for generative modeling.
\newblock In {\em International Conference on Learning Representations}, 2023.

\bibitem{liu2022pseudo}
L.~Liu, Y.~Ren, Z.~Lin, and Z.~Zhao.
\newblock Pseudo numerical methods for diffusion models on manifolds.
\newblock {\em arXiv preprint arXiv:2202.09778}, 2022.

\bibitem{liuflow}
X.~Liu, C.~Gong, et~al.
\newblock Flow straight and fast: Learning to generate and transfer data with rectified flow.
\newblock In {\em The Eleventh International Conference on Learning Representations}, 2023.

\bibitem{lu2022dpm}
C.~Lu, Y.~Zhou, F.~Bao, J.~Chen, C.~Li, and J.~Zhu.
\newblock Dpm-solver: A fast ode solver for diffusion probabilistic model sampling in around 10 steps.
\newblock {\em Advances in Neural Information Processing Systems}, 2022.

\bibitem{lu2025dpm}
C.~Lu, Y.~Zhou, F.~Bao, J.~Chen, C.~Li, and J.~Zhu.
\newblock Dpm-solver++: Fast solver for guided sampling of diffusion probabilistic models.
\newblock {\em Machine Intelligence Research}, pages 1--22, 2025.

\bibitem{luo2023latent}
S.~Luo, Y.~Tan, L.~Huang, J.~Li, and H.~Zhao.
\newblock Latent consistency models: Synthesizing high-resolution images with few-step inference.
\newblock {\em arXiv preprint arXiv:2310.04378}, 2023.

\bibitem{luo2023diff}
W.~Luo, T.~Hu, S.~Zhang, J.~Sun, Z.~Li, and Z.~Zhang.
\newblock Diff-instruct: A universal approach for transferring knowledge from pre-trained diffusion models.
\newblock {\em Advances in Neural Information Processing Systems}, 2023.

\bibitem{meng2023distillation}
C.~Meng, R.~Rombach, R.~Gao, D.~Kingma, S.~Ermon, J.~Ho, and T.~Salimans.
\newblock On distillation of guided diffusion models.
\newblock In {\em Proceedings of the IEEE/CVF conference on computer vision and pattern recognition}, 2023.

\bibitem{midgleyflow}
L.~I. Midgley, V.~Stimper, G.~N. Simm, B.~Sch{\"o}lkopf, and J.~M. Hern{\'a}ndez-Lobato.
\newblock Flow annealed importance sampling bootstrap.
\newblock In {\em International Conference on Learning Representations}, 2023.

\bibitem{nichol2021improved}
A.~Q. Nichol and P.~Dhariwal.
\newblock Improved denoising diffusion probabilistic models.
\newblock In {\em International conference on machine learning}. PMLR, 2021.

\bibitem{ou2024improving}
Z.~Ou, M.~Zhang, A.~Zhang, T.~Z. Xiao, Y.~Li, and D.~Barber.
\newblock Improving probabilistic diffusion models with optimal covariance matching.
\newblock {\em International Conference on Learning Representations}, 2025.

\bibitem{rezende2015variational}
D.~Rezende and S.~Mohamed.
\newblock Variational inference with normalizing flows.
\newblock In {\em International Conference on Machine Learning}. PMLR, 2015.

\bibitem{robbins1992empirical}
H.~E. Robbins.
\newblock An empirical bayes approach to statistics.
\newblock In {\em Breakthroughs in Statistics: Foundations and basic theory}, pages 388--394. Springer, 1992.

\bibitem{rombach2022high}
R.~Rombach, A.~Blattmann, D.~Lorenz, P.~Esser, and B.~Ommer.
\newblock High-resolution image synthesis with latent diffusion models.
\newblock In {\em Proceedings of the IEEE/CVF conference on computer vision and pattern recognition}, 2022.

\bibitem{salimansprogressive}
T.~Salimans and J.~Ho.
\newblock Progressive distillation for fast sampling of diffusion models.
\newblock In {\em International Conference on Learning Representations}, 2022.

\bibitem{salimans2024multistep}
T.~Salimans, T.~Mensink, J.~Heek, and E.~Hoogeboom.
\newblock Multistep distillation of diffusion models via moment matching.
\newblock {\em arXiv preprint arXiv:2406.04103}, 2024.

\bibitem{sohl2015deep}
J.~Sohl-Dickstein, E.~Weiss, N.~Maheswaranathan, and S.~Ganguli.
\newblock Deep unsupervised learning using nonequilibrium thermodynamics.
\newblock In {\em International conference on machine learning}. PMLR, 2015.

\bibitem{song2020denoising}
J.~Song, C.~Meng, and S.~Ermon.
\newblock Denoising diffusion implicit models.
\newblock In {\em International Conference on Learning Representations}, 2021.

\bibitem{songimproved}
Y.~Song and P.~Dhariwal.
\newblock Improved techniques for training consistency models.
\newblock In {\em International Conference on Learning Representations}, 2024.

\bibitem{song2023consistency}
Y.~Song, P.~Dhariwal, M.~Chen, and I.~Sutskever.
\newblock Consistency models.
\newblock In {\em International Conference on Machine Learning}. PMLR, 2023.

\bibitem{song2019generative}
Y.~Song and S.~Ermon.
\newblock Generative modeling by estimating gradients of the data distribution.
\newblock {\em Advances in Neural Information Processing Systems}, 2019.

\bibitem{song2020score}
Y.~Song, J.~Sohl-Dickstein, D.~P. Kingma, A.~Kumar, S.~Ermon, and B.~Poole.
\newblock Score-based generative modeling through stochastic differential equations.
\newblock {\em arXiv preprint arXiv:2011.13456}, 2020.

\bibitem{songscore}
Y.~Song, J.~Sohl-Dickstein, D.~P. Kingma, A.~Kumar, S.~Ermon, and B.~Poole.
\newblock Score-based generative modeling through stochastic differential equations.
\newblock 2021.

\bibitem{Stimper2023}
V.~Stimper, D.~Liu, A.~Campbell, V.~Berenz, L.~Ryll, B.~Schölkopf, and J.~M. Hernández-Lobato.
\newblock normflows: A pytorch package for normalizing flows.
\newblock {\em Journal of Open Source Software}, 8(86):5361, 2023.

\bibitem{tabak2013family}
E.~G. Tabak and C.~V. Turner.
\newblock A family of nonparametric density estimation algorithms.
\newblock {\em Communications on Pure and Applied Mathematics}, 66(2):145--164, 2013.

\bibitem{tabak2010density}
E.~G. Tabak and E.~Vanden-Eijnden.
\newblock Density estimation by dual ascent of the log-likelihood.
\newblock {\em Commun. Math. Sci.}, 8(1):217--233, 2010.

\bibitem{vincent2011connection}
P.~Vincent.
\newblock A connection between score matching and denoising autoencoders.
\newblock {\em Neural computation}, 23(7):1661--1674, 2011.

\bibitem{wang2024prolificdreamer}
Z.~Wang, C.~Lu, Y.~Wang, F.~Bao, C.~Li, H.~Su, and J.~Zhu.
\newblock Prolificdreamer: High-fidelity and diverse text-to-3d generation with variational score distillation.
\newblock {\em Advances in Neural Information Processing Systems}, 2024.

\bibitem{xiao2021tackling}
Z.~Xiao, K.~Kreis, and A.~Vahdat.
\newblock Tackling the generative learning trilemma with denoising diffusion gans.
\newblock {\em arXiv preprint arXiv:2112.07804}, 2021.

\bibitem{xie2024distillation}
S.~Xie, Z.~Xiao, D.~Kingma, T.~Hou, Y.~N. Wu, K.~P. Murphy, T.~Salimans, B.~Poole, and R.~Gao.
\newblock Em distillation for one-step diffusion models.
\newblock {\em Advances in Neural Information Processing Systems}, 2024.

\bibitem{xu2025one}
Y.~Xu, W.~Nie, and A.~Vahdat.
\newblock One-step diffusion models with $ f $-divergence distribution matching.
\newblock {\em arXiv preprint arXiv:2502.15681}, 2025.

\bibitem{yin2024improved}
T.~Yin, M.~Gharbi, T.~Park, R.~Zhang, E.~Shechtman, F.~Durand, and B.~Freeman.
\newblock Improved distribution matching distillation for fast image synthesis.
\newblock {\em Advances in neural information processing systems}, 2024.

\bibitem{yu2024hierarchical}
L.~Yu, T.~Xie, Y.~Zhu, T.~Yang, X.~Zhang, and C.~Zhang.
\newblock Hierarchical semi-implicit variational inference with application to diffusion model acceleration.
\newblock {\em Advances in Neural Information Processing Systems}, 2024.

\bibitem{zhai2024normalizing}
S.~Zhai, R.~Zhang, P.~Nakkiran, D.~Berthelot, J.~Gu, H.~Zheng, T.~Chen, M.~A. Bautista, N.~Jaitly, and J.~Susskind.
\newblock Normalizing flows are capable generative models.
\newblock {\em arXiv preprint arXiv:2412.06329}, 2024.

\bibitem{zhang2019variational}
M.~Zhang, T.~Bird, R.~Habib, T.~Xu, and D.~Barber.
\newblock Variational f-divergence minimization.
\newblock {\em arXiv preprint arXiv:1907.11891}, 2019.

\bibitem{zhang2025towards}
M.~Zhang, W.~Chen, J.~He, Z.~Ou, J.~M. Hern{\'a}ndez-Lobato, B.~Sch{\"o}lkopf, and D.~Barber.
\newblock Towards training one-step diffusion models without distillation.
\newblock {\em arXiv preprint arXiv:2502.08005}, 2025.

\bibitem{zhang2020spread}
M.~Zhang, P.~Hayes, T.~Bird, R.~Habib, and D.~Barber.
\newblock Spread divergence.
\newblock In {\em International Conference on Machine Learning}. PMLR, 2020.

\bibitem{zhou2025inductive}
L.~Zhou, S.~Ermon, and J.~Song.
\newblock Inductive moment matching.
\newblock In {\em International Conference on Machine Learning}, 2025.

\bibitem{zhou2024adversarial}
M.~Zhou, H.~Zheng, Y.~Gu, Z.~Wang, and H.~Huang.
\newblock Adversarial score identity distillation: Rapidly surpassing the teacher in one step.
\newblock In {\em International Conference on Learning Representations}, 2025.

\bibitem{zhou2024score}
M.~Zhou, H.~Zheng, Z.~Wang, M.~Yin, and H.~Huang.
\newblock Score identity distillation: Exponentially fast distillation of pretrained diffusion models for one-step generation.
\newblock In {\em International Conference on Machine Learning}, 2024.

\end{thebibliography}
}

\clearpage
\appendix
\begin{center}
\LARGE
\textbf{Appendix for ``VarDiU: A Variational Diffusive Upper Bound for One-Step Diffusion Distillation"}
\end{center}

\etocdepthtag.toc{mtappendix}
\etocsettagdepth{mtchapter}{none}
\etocsettagdepth{mtappendix}{subsection}
{\small \tableofcontents}

\section{Proof of Variational Upper Bound}
\label{app:bound_proof}

\begin{proposition}[Reverse KL Variational Upper Bound \citep{zhang2019variational}]
\label{prop:var_upper_bound}
For any choice of auxiliary variational distribution \(q^{(t)}_{\phi}(z\mid x_t)\),
\begin{align}
\mathrm{KL}\!\left(p^{(t)}_{\theta}(x_t)\,\|\,p^{(t)}_{d}(x_t)\right)
&\leq \mathrm{KL}\!\left(p^{(t)}_{\theta}(x_t\mid z)p(z)\,\|\,p^{(t)}_{d}(x_t)q^{(t)}_{\phi}(z\mid x_t)\right)
=:\;\mathrm{U}^{(t)}(\theta,\phi).
\end{align}
Moreover, equality holds iff \(q^{(t)}_{\phi}(z\mid x_t)=p^{(t)}_{\theta}(z\mid x_t)\) for 
\(p^{(t)}_{\theta}(x_t)\)-almost every \(x_t\).
\end{proposition}

\begin{proof}
We start from the joint KL:
\begin{align}
\mathrm{KL}&\!\left(p^{(t)}_{\theta}(x_t\mid z)p(z)\,\Big\|\,p^{(t)}_{d}(x_t)\,q^{(t)}_{\phi}(z\mid x_t)\right) \\
&= \iint p^{(t)}_{\theta}(x_t\mid z)p(z)\,
\log\frac{p^{(t)}_{\theta}(x_t\mid z)p(z)}{p^{(t)}_{d}(x_t)\,q^{(t)}_{\phi}(z\mid x_t)}\,\mathrm{d}x_t\,\mathrm{d}z \\
&= \iint p^{(t)}_{\theta}(x_t\mid z)p(z)\,
\big[\log p^{(t)}_{\theta}(x_t\mid z)-\log p^{(t)}_{d}(x_t)\big]\,\mathrm{d}x_t\,\mathrm{d}z \nonumber \\
&\qquad\qquad\qquad\qquad
\qquad+\iint p^{(t)}_{\theta}(x_t\mid z)p(z)\,
\big[\log p(z)-\log q^{(t)}_{\phi}(z\mid x_t)\big]\,\mathrm{d}x_t\,\mathrm{d}z \\
&= \int p^{(t)}_{\theta}(x_t)\,\log\frac{p^{(t)}_{\theta}(x_t)}{p^{(t)}_{d}(x_t)}\,\mathrm{d}x_t
+ \iint p^{(t)}_{\theta}(x_t\mid z)p(z)\,
\log\frac{p^{(t)}_{\theta}(z\mid x_t)}{q^{(t)}_{\phi}(z\mid x_t)}\,\mathrm{d}x_t\,\mathrm{d}z \\
&= \mathrm{KL}\!\left(p^{(t)}_{\theta}(x_t)\,\|\,p^{(t)}_{d}(x_t)\right)
+ \mathbb{E}_{x_t\sim p^{(t)}_{\theta}(x_t)}\!
\left[\mathrm{KL}\!\left(p^{(t)}_{\theta}(z\mid x_t)\,\|\,q^{(t)}_{\phi}(z\mid x_t)\right)\right].
\end{align}

Since the conditional KL is non-negative, it follows that
\begin{align}
\mathrm{KL}\!\left(p^{(t)}_{\theta}(x_t)\,\|\,p^{(t)}_{d}(x_t)\right)
&\leq \mathrm{KL}\!\left(p^{(t)}_{\theta}(x_t\mid z)p(z)\,\|\,p^{(t)}_{d}(x_t)q^{(t)}_{\phi}(z\mid x_t)\right) =:\;\mathrm{U}^{(t)}(\theta,\phi).
\end{align}
Equality holds iff \(q^{(t)}_{\phi}(z\mid x_t)=p^{(t)}_{\theta}(z\mid x_t)\) 
for \(p^{(t)}_{\theta}(x_t)\) almost every \(x_t\).
\end{proof}

\section{Proof of VarDiU Gradient Estimator}
\label{app:proofgrad}
\begin{proposition}
The gradient can be estimated by
\begin{align}
    \nabla_{{\theta}} \int p^{(t)}_{\theta}(x_t) \log p^{(t)}_d(x_t) \mathrm{d}x_t= \nabla_{{\theta}} \int p^{(t)}_\theta(x_t) \left(x_t^\top [\nabla_{x_t}\log p^{(t)}_d(x_t)]_{\text{sg}}\right) \mathrm{d}x_t,
\label{eq:grad}
\end{align}
where $[\,\cdot\,]_{\mathrm{sg}}$ denotes the stop-gradient operator.
\end{proposition}
\begin{proof}
Since \(x_t\) is generated via a differentiable transformation
\begin{equation}
    x_t = \mathcal{F}_{\theta}(g_{\theta}(z), \bepsilon,t) = g_{\theta}(z) + \bepsilon \sigma_t,
\end{equation}
where \(\bepsilon\) is drawn from a noise distribution \(p_{\bepsilon}(\bepsilon)\) that is independent of \(\theta\).
Then we have
\begin{align}
     & \iint p^{(t)}_{\theta}(x_t|z)p(z) \log p^{(t)}_d(x_t)\,\mathrm{d}x_t\,\mathrm{d}z 
     \nonumber\\
     & =\iint p(z)p_{\bepsilon}(\bepsilon) \log p^{(t)}_d\Bigl(\mathcal{F}_{\theta}(g_{\theta}(z), \bepsilon,t)\Bigr) \mathrm{d}\bepsilon \mathrm{d}z= \mathbb{E}_{z\sim p(z),\,\bepsilon\sim p_{\bepsilon}(\bepsilon)}\Bigl[\log p^{(t)}_{d}\Bigl(\mathcal{F}_{\theta}(g_{\theta}(z), \bepsilon,t)\Bigr)\Bigr]
\end{align}
Since \(\mathcal{F}_{\theta}(g_{\theta}(z), \bepsilon,t)\) is differentiable in $\theta$ and under appropriate regularity conditions, we have
\begin{equation}
   \nabla_{{\theta}} \iint p^{(t)}_{\theta}(x_t|z)p(z) \log p^{(t)}_d(x_t) \mathrm{d}z \mathrm{d}x_t  = \mathbb{E}_{z\sim p(z),\,\bepsilon\sim p_{\bepsilon}(\bepsilon)}\Biggl[\nabla_{\theta} \log p^{(t)}_d\bigl(\mathcal{F}_{\theta}(g_{\theta}(z), \bepsilon,t)\bigr)\Biggr]
\end{equation}
Using the chain rule, the gradient inside the expectation becomes
\begin{align}
    &\nabla_{\theta} \log p^{(t)}_d\bigl(\mathcal{F}_{\theta}(g_{\theta}(z), \bepsilon,t)\bigr) = \nabla_{x_t} \log p^{(t)}_d(x_t)\bigg|_{x_t=\mathcal{F}_{\theta}(g_{\theta}(z), \bepsilon,t)} \cdot \frac{\partial \mathcal{F}_{\theta}(g_{\theta}(z), \bepsilon,t)}{\partial \theta}
\end{align}
Substituting back into our expression for the gradient, we obtain
\begin{align}
    &\nabla_{{\theta}} \iint p^{(t)}_{\theta}(x_t|z)p(z) \log p^{(t)}_d(x_t) \mathrm{d}x_t\mathrm{d}z  \nonumber\\
    &= \mathbb{E}_{z\sim p(z),\,\bepsilon\sim p_{\bepsilon}(\bepsilon)}\Biggl[\nabla_{x_t} \log p^{(t)}_d(x_t)\bigg|_{x_t=\mathcal{F}_{\theta}(g_{\theta}(z), \bepsilon,t)}\cdot\frac{\partial \mathcal{F}_{\theta}(g_{\theta}(z), \bepsilon,t)}{\partial \theta}\Biggr]\\
    &=\mathbb{E}_{\substack{z \sim p({z}), x_t \sim p^{(t)}_{\theta}(x_t|z)}}\Biggl[\nabla_{x_t} \log p^{(t)}_d(x_t)\frac{\partial x_t}{\partial \theta}\Biggr] = \iint p^{(t)}_{\theta}(x_t|z)p(z) \nabla_{x_t} \log p^{(t)}_d(x_t)\frac{\partial x_t}{\partial {\theta}}\mathrm{d}x_t\mathrm{d}z ,
\end{align}
as required.
\end{proof}

\section{A Practically Equivalent Variational Diffusive Upper Bound Loss}
\label{eq:eqloss}

\paragraph{Tweedie’s formula~\cite{efron2011tweedie,robbins1992empirical}.}
For the Gaussian corruption kernel \(p_t(x_t\mid  x)=\mathcal{N}(x_t; x,\sigma_t^2\mathbf I)\),
Tweedie’s identity gives
\begin{equation}
\nabla_{x_t}\log p^{(t)}_d(x_t)
=\frac{\mu(x_t;t)-x_t}{\sigma_t^2},
\qquad
\mu(x_t;t):=\E_{p_d( x\mid x_t)}[ x\mid x_t].
\label{eq:tweedie}
\end{equation}

We have from \cref{eq:diuexpand}
\begin{align}
    &\mathcal{L}({\phi},{\theta}) = - \int_{0}^1 \omega(t)p_\theta^{(t)}(x_t,z)\bigg[x_t^\top [\nabla_{x_t}\log p^{(t)}_d(x_t)]_{\text{sg}} + \log q^{(t)}_{\phi}(z|x_t;t)\bigg]\mathrm{d}t + const
    \nonumber \\
    &=-\int_{0}^1 \omega(t)p_\theta^{(t)}(x_t,z)\bigg[{g}_{\theta}(z)^\top \left[\frac{{\mu}(x_t;t)-{g}_{\theta}(z)}{\sigma_t^2}\right]_{\text{sg}} + \log q^{(t)}_{\phi}(z|x_t;t)\bigg]\mathrm{d}t + const.
\end{align}
\begin{proof}
We are given the VarDiU objective:
\begin{align}
\mathcal{L}(\phi,\theta)
&= - \int_{0}^1 \!\omega(t)\;p_\theta^{(t)}(x_t,z)\!\Big[x_t^\top \big[\nabla_{x_t}\log p^{(t)}_d(x_t)\big]_{\mathrm{sg}} + \log q^{(t)}_{\phi}(z\mid x_t;t)\Big]\mathrm{d}t + \mathrm{const}.
\end{align}
With the stop-gradient applied to the score term,
\begin{align}
x_t^\top \big[\nabla_{x_t}\log p^{(t)}_d(x_t)\big]_{\mathrm{sg}}
&= \frac{1}{\sigma_t^2}\,x_t^\top\big[\mu(x_t;t)-x_t\big]_{\mathrm{sg}} \nonumber\\
&= \frac{1}{\sigma_t^2}\Big\{ x^\top\big[\mu(x_t;t)- x\big]_{\mathrm{sg}}
+ (x_t- x)^\top\big[\mu(x_t;t)-x_t\big]_{\mathrm{sg}}\Big\}.
\label{eq:split}
\end{align}
Write \(x_t= x+\sigma_t\varepsilon\) with \(\varepsilon\sim\mathcal N(0,\mathbf I)\), and note that
\(\big[\mu(x_t;t)-x_t\big]_{\mathrm{sg}}\) is treated as a constant w.r.t. \((\phi,\theta)\).
Taking the expectation over \(x_t\mid  x\) of \eqref{eq:split} yields
\begin{align}
\E_{x_t\mid  x}\!\Big[x_t^\top \big[\nabla_{x_t}&\log p^{(t)}_d(x_t)\big]_{\mathrm{sg}}\Big]\nonumber\\
=& \frac{1}{\sigma_t^2}\,\E_{x_t\mid  x}\!\Big[ x^\top\big[\mu(x_t;t)- x\big]_{\mathrm{sg}}\Big]
+ \frac{1}{\sigma_t^2}\,\E_{x_t\mid  x}\!\Big[(x_t- x)^\top\big[\mu(x_t;t)-x_t\big]_{\mathrm{sg}}\Big].
\label{eq:two-terms}
\end{align}
The first term on the right is exactly
\( x^\top\big[\big(\mu(x_t;t)- x\big)/\sigma_t^2\big]_{\mathrm{sg}}\).
The second term decomposes as
\[
\E_{x_t\mid  x}\!\Big[(x_t- x)^\top\big[\mu(x_t;t)\big]_{\mathrm{sg}}\Big]
-\E_{x_t\mid  x}\!\big[\|x_t- x\|^2\big].
\]
Because \((x_t- x)=\sigma_t\varepsilon\) with \(\varepsilon\) independent of \((\phi,\theta)\) and
\(\big[\mu(\cdot;t)\big]_{\mathrm{sg}}\) blocks all gradients through its argument,
this entire bracket has \emph{zero gradient} w.r.t. \((\phi,\theta)\); in particular,
\(\E_{x_t\mid x}\!\big[\|x_t- x\|^2\big]=d\,\sigma_t^2\) (with \(d\) the data dimension) and
\(\E_{x_t\mid  x}\!\big[(x_t- x)^\top[\mu(x_t;t)]_{\mathrm{sg}}\big]\) does not contribute
to parameter gradients under reparametrisation. Hence it can be absorbed into ``\(\mathrm{const}\)''.

Combining, we obtain (up to a parameter-independent constant)
\begin{align}
\E_{x_t\mid  x}\!\Big[x_t^\top \big[\nabla_{x_t}\log p^{(t)}_d(x_t)\big]_{\mathrm{sg}}\Big]
=  x^\top \Big[\frac{\mu(x_t;t)- x}{\sigma_t^2}\Big]_{\mathrm{sg}} + \mathrm{const}.
\end{align}
Substituting \( x={g}_\theta(z)\) and keeping the \(\log q^{(t)}_{\phi}(z\mid x_t;t)\) term unchanged gives
\begin{align}
\mathcal{L}(\phi,\theta)
&= -\int_{0}^1 \omega(t)\;
p_\theta^{(t)}(x_t,z)
\Big[{g}_{\theta}(z)^\top \Big[\frac{{\mu}(x_t;t)-{g}_{\theta}(z)}{\sigma_t^2}\Big]_{\mathrm{sg}}
+ \log q^{(t)}_{\phi}(z\mid x_t;t)\Big]\mathrm{d}t
+ \mathrm{const}, 
\end{align}
as required.  
\end{proof}

\section{Joint Entropy is Constant}
\label{sec:entropy_proof}

\begin{proposition}
\label{prop:joint}
Let 
\begin{equation}
p^{(t)}_{\theta}(x,z)=p(z)\,p_{\theta}^{(t)}(x| z)
\end{equation}
with a fixed prior \(p(z)\) and conditional
\(\mathcal N\bigl(x;\,g_{\theta}(z),\sigma^2_t{I}\bigr)\).  Then the joint entropy
\(\mathbb{H}(p_\theta^{(t)}(x_t,z))\) is independent of \(\theta\).
\end{proposition}

\begin{proof}
First, write the joint entropy as the sum of marginal and conditional entropies:
\begin{equation}
\mathbb{H}(p_{\theta}^{(t)}(x_t,z))
=-\iint p^{(t)}_{\theta}(x_t,z)\,\log p^{(t)}_{\theta}(x_t,z)\,\mathrm{d}x_t\,\mathrm{d}z
=\mathbb{H}(p(z))+\mathbb{H}(p_{\theta}^{(t)}(x_t|z)),
\end{equation}
where
\begin{equation}
\mathbb{H}(p(z))=-\!\int p(z)\,\log p(z)\,\mathrm{d}z,\\
\mathbb{H}(p_{\theta}^{(t)}(x_t|z))
=-\!\iint p(z)\,p^{(t)}_{\theta}(x_t| z)\,\log p^{(t)}_{\theta}(x_t| z)\,\mathrm{d}x_t\,\mathrm{d}z.
\end{equation}

By assumption, the prior \(p(z)\) does not depend on \(\theta\), so
\(\mathbb{H}(p(z))\) is a fixed number.

For each \(z\), \(x_t| z\sim\mathcal N(g_{\theta}(z),\sigma_t^2I)\).  The entropy of a Gaussian
depends only on its covariance:
\begin{equation}
\mathbb{H}\bigl(\mathcal N(\mu,\Sigma)\bigr)
=\tfrac12\log\bigl[(2\pi e)^d\det\Sigma\bigr].
\end{equation}
where $d$ is independent of $\theta$, Here \(\Sigma=\sigma_t^2I\), so
\begin{equation}
\mathbb{H}\bigl(\mathcal N(g_{\theta}(z),\sigma_t^2I)\bigr)
=\tfrac{d}{2}\,\log\bigl[(2\pi e)\sigma^2\bigr],
\end{equation}
which is independent of \(g_{\theta}(z)\).  Hence
\begin{equation}
\mathbb{H}(p_{\theta}^{(t)}(x_t|z))
=\!\int p(z)\,\mathbb{H}\bigl(\mathcal N(g_{\theta}(z),\sigma_t^2I)\bigr)\,\mathrm{d}z
=\tfrac{d}{2}\,\log\bigl[(2\pi e)\sigma_t^2\bigr],
\end{equation}
also constant.
Since both \(\mathbb{H}(p(z))\) and \(\mathbb{H}(p_{\theta}^{(t)}(x_t|z))\) are independent of \(\theta\), their sum
\(\mathbb{H}(p_\theta^{(t)}(x_t,z))\) is a constant.
\end{proof}

\section{Proof of \cref{thm:entropy}}
\label{app:entropy}
\begin{theorem}[Restatement of \cref{thm:entropy}]
By the chain rule of entropy, we have:
$\mathbb{H}(p^{(t)}_\theta(x_t, z)) = \mathbb{H}(p^{(t)}_\theta(z | x_t)) + \mathbb{H}(p^{(t)}_\theta(x_t))$.
Since \( \mathbb{H}(p^{(t)}_\theta(x_t, z)) \) is constant with respect to \(\theta\), it follows that:
\begin{align}
\max_{\theta} \mathbb{H}(p^{(t)}_\theta(x_t))
= \min_{\theta} \mathbb{H}(p^{(t)}_\theta(z|x_t))
\le \min_{\theta, \phi} \;
- \iint p^{(t)}_{\theta}(x_t, z) \log q^{(t)}_{\phi}(z | x_t) \,\mathrm{d}z\,\mathrm{d}x_t,
\end{align}
where \( q^{(t)}_{\phi}(z|x_t) \) is a variational approximation to the true posterior \( p^{(t)}_\theta(z|x_t) \).
\end{theorem}
\begin{proof}
We have 
\begin{align}
 \mathbb{H}(p^{(t)}_{\theta}(z|x_t))&=-\iint \log p_{\theta}(z|x_t) p_{\theta}(z|x_t)p_{\theta}^{(t)}(x_t)\mathrm{d} z\mathrm{d}x_t\\
    &=-\iint \log \left(\frac{p_{\theta}(z|x_t)}{q^{(t)}_{\phi}(z|x_t)}q^{(t)}_{\phi}(z|x_t) \right) p^{(t)}_{\theta}(x_t|z)p(z)\mathrm{d} z\mathrm{d}x_t\\
    &= -\iint \log \frac{p_{\theta}(z|x_t)}{q^{(t)}_{\phi}(z|x_t)}  p^{(t)}_{\theta}(x_t|z)p(z)\mathrm{d} z\mathrm{d}x_t-\iint \log q^{(t)}_{\phi}(z|x_t) p^{(t)}_{\theta}(z|x_t)p^{(t)}_{\theta}(x_t)\mathrm{d} z\mathrm{d}x_t\\
    &=-\underbrace{\mathrm{KL}(p_{\theta}(z|x_t)||q_{\phi}^{(t)}(z|x_t))}_{\geq 0}-\int \log q^{(t)}_{\phi}(z|x_t) p_{\theta}(z|x_t)p^{(t)}_{\theta}(x_t)\mathrm{d} z \mathrm{d}x_t.
\end{align}
Therefore, 
$$\mathbb{H}(p^{(t)}_{\theta}(z|x_t))\leq -\iint \log q^{(t)}_{\phi}(z|x_t) p_{\theta}(z|x_t)p_{\theta}^{(t)}(x_t)\mathrm{d} z\mathrm{d}x_t$$ and the bound is tight when $\mathrm{KL}(p_{\theta}(z|x_t)||q^{(t)}_{\phi}(z|x_t))=0$.
To minimise $\mathrm{KL}(p_{\theta}(z|x_t)||q^{(t)}_{\phi}(z|x_t))$ wrt $\phi$, we have
\begin{align}
    &\min_{\phi} \mathrm{KL}(p_{\theta}^{(t)}(z|x_t)||q^{(t)}_{\phi}(z|x_t))=  \min_{\phi} \iint \log q_{\phi}(z|x_t)p^{(t)}_{\theta}(x_t|z)p(z)\mathrm{d} z\mathrm{d}x_t.\nonumber
    \\
    \Rightarrow&\max_{\theta} \mathbb{H}(p_\theta^{(t)}(x_t))= \min_{\theta} \mathbb{H}(p^{(t)}_{\theta}(z|x_t)) \leq \min_{\theta} \min_{\phi} -\iint \log q^{(t)}_{\phi}(z|x_t) p(z)p^{(t)}_{\theta}(x_t|z)\mathrm{d} z\mathrm{d}x_t
\end{align}
where the last line since Proposition \ref{prop:joint}
\begin{align}
    \underbrace{\mathbb{H}(p_\theta^{(t)}(x_t,z))}_{\mathrm{const.}}=\mathbb{H}(p_\theta^{(t)}(x_t))+\mathbb{H}(p^{(t)}_{\theta}(z|x_t))
    \nonumber 
    \Rightarrow \mathbb{H}(p_\theta^{(t)}(x_t)) = -\mathbb{H}(p^{(t)}_{\theta}(z|x_t))+\mathrm{const.}
\end{align}
\end{proof}
\section{Experimental Details}
\label{app:exp_details}
In this section, we present the details of the experimental setting.
\subsection{Implementation Details}
\label{app:imp_detail}
For the generator $g_{\theta}$, we use a 5-layer MLP with latent dimension 2, hidden dimension 400 and SiLU activation. We use Adam to train the generator with learning rate $10^{-4}$, batch size $1024$, and gradient norm clip $10.0$.

We employ the same 5-layer MLP with hidden dimension 400, and SiLU activation as well to predict the mean and log-variance of the Gaussian posterior. We use Adam to train the generator and posterior network with learning rate $10^{-4}$, batch size $1024$, and gradient norm clip $10.0$.

For Diff-instruct \cite{luo2023diff}, we use the same 5-layer MLP with hidden dimension 400, and SiLU activation to parametrise the student score network. We use Adam to train the student score network with learning rate $5*10^{-5}$, batch size $1024$, and gradient norm clip $10.0$.

Both methods are trained for $1{,}000{,}000$ epochs until convergence. For the learned score, we train an EDM~\citep{karras2022edm} on $10{,}000$ observed samples, while the empirical score is estimated from $10{,}000$ observed samples. All experiments are conducted on \texttt{NVIDIA RTX 3090} GPUs.  

\paragraph{Maximum Mean Discrepancy Implementation Details}
For evaluation, we employ the Maximum Mean Discrepancy (MMD)~\citep{gretton2012kernel} with five Gaussian kernels of bandwidths $\{2^{-2}, 2^{-1}, 2^{0}, 2^{1}, 2^{2}\}$ for evaluation across multiple scales. 
Given two sets of samples $\{x_i\}_{i=1}^n \sim p_d$ and $\{y_j\}_{j=1}^m \sim p_\theta$, the squared MMD is defined as  
\begin{align}
    \mathrm{MMD}^2(p_d, p_\theta; k) 
    &= \mathbb{E}_{x,x' \sim p_d}[k(x,x')] 
    + \mathbb{E}_{y,y' \sim p_\theta}[k(y,y')] 
    - 2 \mathbb{E}_{x \sim p_d, y \sim p_\theta}[k(x,y)],
\end{align}
where $k(\cdot,\cdot)$ is the positive definite kernel. A smaller MMD indicates closer alignment between the generated and real distributions.  For the results in Table 1, we perform 10 independent training runs. Evaluation metrics are recorded every $1{,}000$ epochs, and at each evaluation step, we compute the mean and standard deviation across runs. We use $10,000$ samples for MMD evaluation. To summarise performance, we report the average of these statistics over the last $50$ evaluations.

\paragraph{VarDiU Posterior Parametrisation}
For VarDiU, we parametrise the Gaussian posterior as
\begin{equation}
     q_{\phi}^{(t)}(z| x_t;t)
=\mathcal{N}\Bigl(z;\,\mu_{\phi}(x_t;t),\,
\sigma_t^2\,\mathrm{diag}\!\bigl(\sigma_{\phi}^2(x_t;t)\bigr)\Bigr).
\end{equation}
The time-dependent scaling factor \( \sigma_t \) modulates the entropy of the base distribution to match the noise level at diffusion step \( t \). This allows the posterior to capture fine-grained detail at small \( t \), while maintaining flexibility to denoise highly corrupted samples at larger \( t \).

For flow-based posterior, we implement Neural Splines Flows \cite{durkan2019neural} using an open-source package \texttt{normflows}\footnote{https://github.com/VincentStimper/normalizing-flows} \cite{Stimper2023}, we choose the flow length to be 4. We use Adam to train the generator and flow-based posterior network with learning rate $10^{-4}$, batch size $1024$, and gradient norm clip $10.0$ as well.

\subsection{Noise Weighting and Scheduling}
\label{app:schedule}

\begin{figure}
    \centering
    \includegraphics[width=0.5\linewidth]{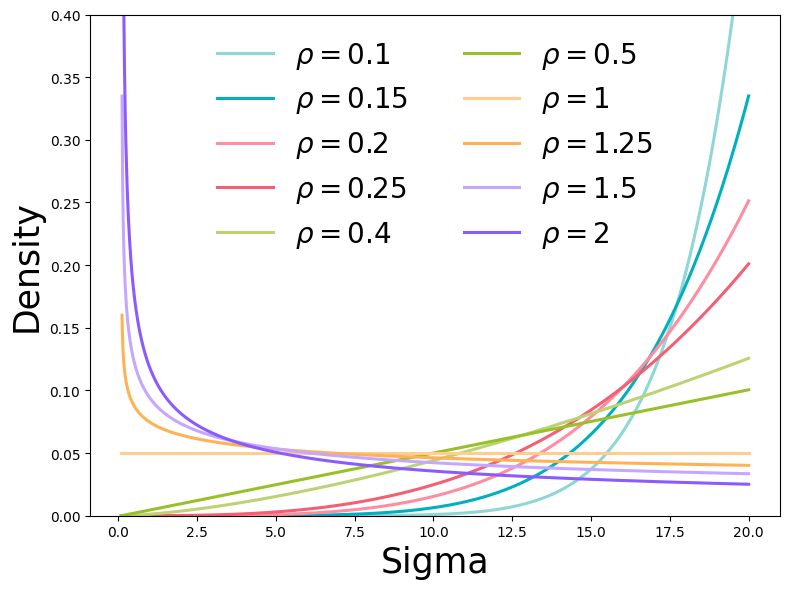}
    \caption{Density of sigma v.s different $\rho$. Here $\sigma_{\min}=0.1,\sigma_{\max}=20$.}
    \label{fig:sigma}
\end{figure}
The one-step generation procedure is given by
\begin{equation}
     x = \int \delta(x-g_{\theta}(z))\mathrm{d}x,
\end{equation}
where the latent prior is sampled as $z \sim \mathcal{N}(\mathbf{0}, \sigma^2_{\text{init}} \bI).$ We set $\sigma^2_{\text{init}} = 1$ in all toy data experiments.

For the forward diffusion distribution, we set
\begin{equation}
    q(x_t|x_0) = \mathcal{N}(x_t;\alpha_t x_0,\sigma_t^2 I)
\end{equation}
Unlike Variance Preserving (VP;\cite{song2019generative,ho2020denoising}) schedule which must satisfy $\alpha_t^2 + \sigma_t^2 =1$, we employ the schedule used in EDM~\cite{karras2022edm} which is a popular class of diffusion models, where we set $\alpha_t=1$ for all $t$ and we sample the time variable $t \sim \mathrm{Unif}[0,1]$ and define the corresponding noise levels as
\begin{equation}
\label{eq:schedule}
    \sigma_t = \sigma_{\min} + t^{\rho} \, (\sigma_{\max} - \sigma_{\min}).
\end{equation}
A visualisation of the resulting density of $\sigma_t$ under different values of $\rho$ is provided in \cref{fig:sigma}.  
For Diff-Instruct, we fix $\rho = 1.5$ across all three score estimation settings. We found an annealing $\rho$ does not help with the training dynamics of Diff-Instruct, so we use a fixed schedule.
\paragraph{Annealing $\rho$.}  
In  {VarDiU}, we adopt an annealing strategy in which $\rho$ gradually increases with training iterations. Specifically, $\rho$ is initialised at $\rho_{\text{init}} = 0.1$ and incremented by $0.01$ every $1000$ epochs until it reaches $\rho_{\text{end}}$. 

The annealed noise schedule plays a crucial role in stabilising training. The central idea is to initially place more emphasis on larger noise levels ($\sigma$), where the learning problem is smoother and less sensitive to estimation errors. As training progresses, the schedule gradually shifts towards smaller $\sigma$ values, allowing the model to refine details and capture higher-frequency structures. This coarse-to-fine strategy not only mitigates instability at the early stage but also improves convergence and sample quality in later stages.

\paragraph{(1) Given The True Score.}  
We set $\sigma_{\min} = 0.1$ and $\sigma_{\max} = 20$ for both  {Diff-Instruct} and  {VarDiU}.  
For  {VarDiU}, we set $\rho_{\text{end}} = 5.0$.
\paragraph{(2) Given A Pre-trained DM.}  
For  {Diff-Instruct}, we set $\sigma_{\min} = 1.1$ and $\sigma_{\max} = 40$.  
For  {VarDiU}, we set $\sigma_{\min} = 1.5$, $\sigma_{\max} = 40$, and $\rho_{\text{end}} = 2.0$.
\paragraph{(3) Given A Dataset.}  
We employ $10{,}000$ samples for empirical score estimation, with details provided in \cref{eq:emp_score}.  
For both  {Diff-Instruct} and  {VarDiU}, we set $\sigma_{\min} = 0.65$ and $\sigma_{\max} = 40$, and for  {VarDiU}, we additionally set $\rho_{\text{end}} = 2.0$.
\paragraph{Noise Weighting.}  
For both  {Diff-Instruct} and  {VarDiU}, we adopt the weighting function
\[
    \omega(t) = \frac{\sigma_t^2}{\sigma_{\max}^2}.
\]
The noise schedule aligns with \cite{zhang2025towards,luo2023diff}. We further divide the weighting by the $\sigma_{\max}$ for not changing the learning rate of training.
\subsection{Empirical Score Estimation}
\label{eq:emp_score}
\paragraph{Score of a Gaussian-smoothed empirical distribution.}
Let $\{x_0^{(n)}\}_{n=1}^N \subset \mathbb{R}^d$ be data points, and define the Gaussian kernel
\[
\phi_\sigma(x) := (2\pi\sigma^2)^{-d/2}
\exp\!\left(-\tfrac{\|x\|^2}{2\sigma^2}\right).
\]
The $\sigma$-smoothed empirical distribution is
\[
\hat p_{\theta}(x)
:= \frac{1}{N}\sum_{n=1}^N \mathcal{N}\!\bigl(x | x_0^{(n)}, \sigma^2 I\bigr)
= \frac{1}{N}\sum_{n=1}^N \phi_\sigma(x - x_0^{(n)}).
\]

\begin{lemma}[Log-sum gradient identity]
For positive functions $\{a_n(x)\}_{n=1}^N$,
\[
\nabla_{x} \log\!\Bigl(\sum_{n=1}^N a_n(x)\Bigr)
= \sum_{n=1}^N w_n(x)\,\nabla_{x} \log a_n(x),
\qquad
w_n(x) := \frac{a_n(x)}{\sum_{m=1}^N a_m(x)}.
\]
\end{lemma}

\begin{proof}
By the chain and quotient rules,
\[
\nabla_{x} \log\!\Bigl(\sum_{n} a_n\Bigr)
= \frac{\sum_n \nabla_{x} a_n}{\sum_m a_m}
= \sum_n \frac{a_n}{\sum_m a_m}\,\frac{\nabla_{x} a_n}{a_n}
= \sum_n w_n\,\nabla_{x} \log a_n.\qedhere
\]
\end{proof}

Applying the lemma with $a_n(x)=\phi_\sigma(x-x_0^{(n)})$ gives
\[
\nabla_{x} \log \hat p_{\theta}(x)
= \sum_{n=1}^N w_n(x;\sigma)\,\nabla_{x} \log \phi_\sigma(x-x_0^{(n)}),
\]
with responsibilities
\[
w_n(x;\sigma)
= \frac{\exp\!\bigl(-\|x-x_0^{(n)}\|^2/(2\sigma^2)\bigr)}
{\sum_{m=1}^N \exp\!\bigl(-\|x-x_0^{(m)}\|^2/(2\sigma^2)\bigr)}.
\]

The Gaussian score is
\[
\nabla_{x} \log \phi_\sigma(x-x_0^{(n)})
= -\frac{1}{\sigma^2}\bigl(x-x_0^{(n)}\bigr)
= \frac{x_0^{(n)}-x}{\sigma^2}.
\]

Therefore,
\[
\nabla_{x} \log \hat p_{\theta}(x)
= \sum_{n=1}^N w_n(x;\sigma)\,\frac{x_0^{(n)}-x}{\sigma^2}
= \frac{\sum_{n=1}^N w_n(x;\sigma)\,x_0^{(n)} - x}{\sigma^2}.
\]
\textbf{Final form.}
For any query point $x_t$, the score of the Gaussian–KDE is
\begin{equation}
    \;\nabla_{x_t} \log \hat p_{\theta}(x_t)
= \frac{\bar{x}(x_t;\sigma) - x_t}{\sigma^2} \; ; \qquad\qquad\qquad
    \bar{x}(x;\sigma) := \sum_{n=1}^N w_n(x;\sigma)\,x_0^{(n)}.
    \label{eq:empscore}
\end{equation}
We employ \cref{eq:empscore} to compute the empirical score using the given dataset practically. The accuracy of empirical score estimation tends to be small when a large number of samples are given, providing useful teacher guidance.

\subsection{Symmetric Sampling}
To reduce the variance of the stochastic estimator, we employ a symmetric sampling scheme \cite{botev2017variance,kroese2013handbook} for both VarDiU and Diff-Instruct.

We first draw a half-batch of latent codes,
\[
    z^{(i)} \sim \mathcal{N}(\mathbf{0}, \sigma_{\text{init}}\mathbf{I}), 
    \quad i = 1, \dots, \tfrac{B}{2},
\]
where $B$ is the batch size.  
We then duplicate the samples to construct a full batch,
\[
    z = \bigl(z^{(1)}, \dots, z^{(B/2)}, \; z^{(1)}, \dots, z^{(B/2)} \bigr).
\]
Passing these latents through the generator yields samples
\[
    x_0 = g_{\theta}(z).
\]
Next, we simulate forward diffusion process
\[
    \bepsilon^{(i)} \sim \mathcal{N}(\mathbf{0}, \mathbf{I}), 
    \quad i = 1, \dots, \tfrac{B}{2},
\]
and extend them symmetrically as
\[
    \bepsilon = \bigl(\bepsilon^{(1)}, \dots, \bepsilon^{(B/2)}, \; -\bepsilon^{(1)}, \dots, -\bepsilon^{(B/2)} \bigr).
\]
Thus, each noise vector is paired with its negative, ensuring that the perturbations are exactly centred.
We also sample noise levels from the power distribution
\[
    \sigma^{(i)} \sim p_\rho(\sigma), \quad i = 1, \dots, \tfrac{B}{2},
\]
and replicate them symmetrically:
\[
    \bsigma = \bigl(\sigma^{(1)}, \dots, \sigma^{(B/2)}, \; \sigma^{(1)}, \dots, \sigma^{(B/2)} \bigr).
\]

\textbf{Noisy inputs.}  
The perturbed inputs are then constructed as
\[
    x_t = x_0 + \bsigma \odot \bepsilon,
\]
where $\odot$ denotes element-wise multiplication.

We found symmetric sampling can reduce the variance during training, further enhance the training stability of VarDiU, and can help with stabilising the training with Diff-Instruct.
\section{Related Works}\label{app:related}
Diffusion acceleration has received significant attention as a means of reducing the number of steps required in the reverse process. At the core of this problem lies the posterior distribution $q(x_0 \mid x_t)$, which connects a noisy intermediate $x_t$ to the original data $x_0$. Accurate modelling of this posterior enables faithful reconstruction of $x_0$ in fewer steps, thereby lowering the sampling cost. Existing approaches can be broadly grouped into strategies such as designing alternative samplers~\cite{lu2022dpm,lu2025dpm,song2020denoising,de2021diffusion}, relaxing the Gaussian posterior assumption~\cite{debortoli2025distributionaldiffusionmodelsscoring,lipmanflow,liuflow}, or projecting the dynamics onto lower-dimensional subspaces~\cite{rombach2022high}. Whilst these techniques substantially reduce sampling time, they still typically require a non-trivial number of steps to achieve high-quality generation.  

Another important direction focuses on compressing the reverse diffusion process into shorter chains. This includes progressive distillation~\cite{salimansprogressive,meng2023distillation}, consistency models and their improvements~\cite{song2023consistency,songimproved,kim2023consistency}, as well as extensions to latent diffusion~\cite{luo2023latent,rombach2022high}. More recently, Inductive Moment Matching has been proposed as a single-stage framework that trains few-step models from scratch by matching distributional moments~\cite{zhou2025inductive}.  

In parallel, several works have investigated single-step diffusion models. Some fine-tune pre-trained diffusion models by minimising divergences such as the Diffusive KL or Fisher divergence~\cite{luo2023diff,xie2024distillation,zhou2024score,xu2025one}, while others incorporate adversarial losses to enhance perceptual fidelity~\cite{yin2024improved,zhou2024adversarial}. Our work builds on this line by eliminating the need for a separate student score network, thereby avoiding bias from inaccurate score estimation and enabling more stable training with improved generation quality.  

Finally, recent work has shown that teacher guidance is not strictly necessary for effective diffusion distillation, highlighting that the multi-scale representations learned by diffusion models are the crucial factor behind the success of one-step generators~\cite{zhang2025towards}.

\newpage
\section{More Convergence and Efficiency Results}
We present the log-MMD plots with respect to both generator gradient iterations and total training time, using the pre-trained DM and the given dataset in \cref{fig:speed_fake,fig:speed_emp}, respectively.

When relying on the pre-trained DM, both methods exhibit high variance and unstable curves due to inaccurate score predictions from the pre-trained model. The noticeable jump of the VarDiU variants in the middle of training arises from the annealing schedule, where the density concentrates on a range of $\sigma$ values for which the pre-trained DM suffers from large estimation errors. Nevertheless, the VarDiU variants maintain relatively low variance in the log-density curves, highlighting their robustness.

In contrast, when using the given dataset, we estimate the score empirically. The empirical score becomes increasingly accurate with larger sample sizes. As shown in \cref{fig:speed_emp}, the VarDiU variants achieve performance comparable to that obtained with the true analytical score. Furthermore, VarDiU demonstrates substantially improved training stability and efficiency when employing a Gaussian-parameterised posterior.

Overall, these results highlight the superiority of VarDiU, which consistently achieves more stable and efficient convergence across different settings compared to the baseline.
\begin{figure}[h]
    \centering
    \includegraphics[width=\linewidth]{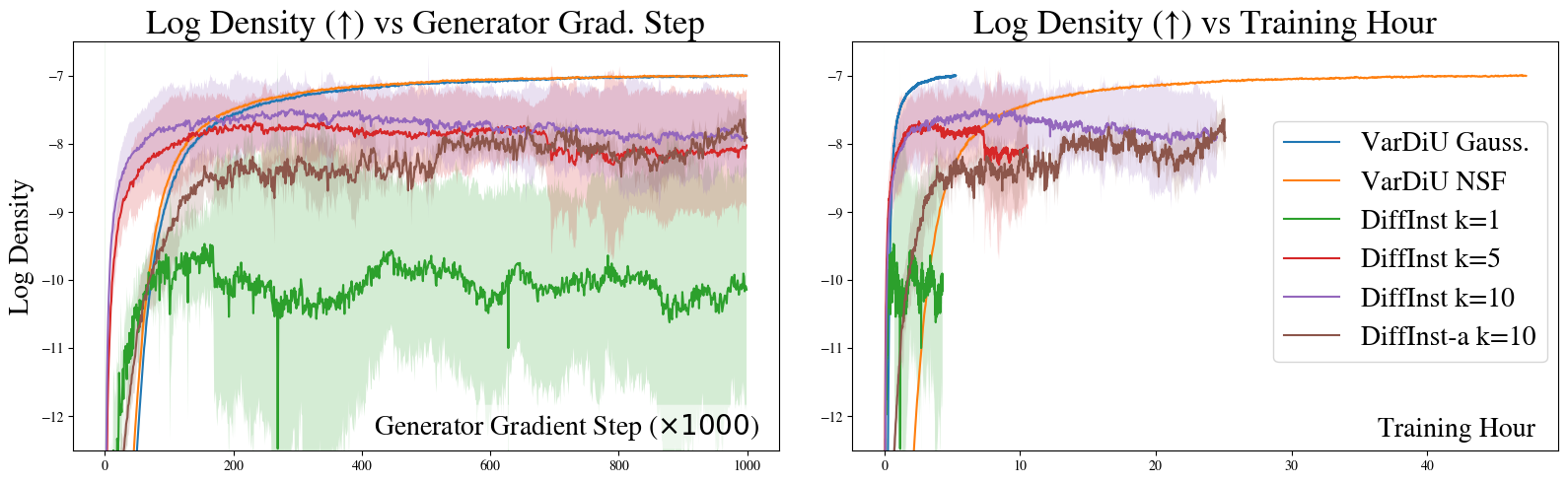}
    \caption{Log-density trajectory with true score seeting.}
    \label{fig:true_scorelogpd}
\end{figure}

\begin{figure}[h]
    \centering
    \begin{subfigure}[b]{0.98\textwidth}
        \centering
        \includegraphics[width=\linewidth]{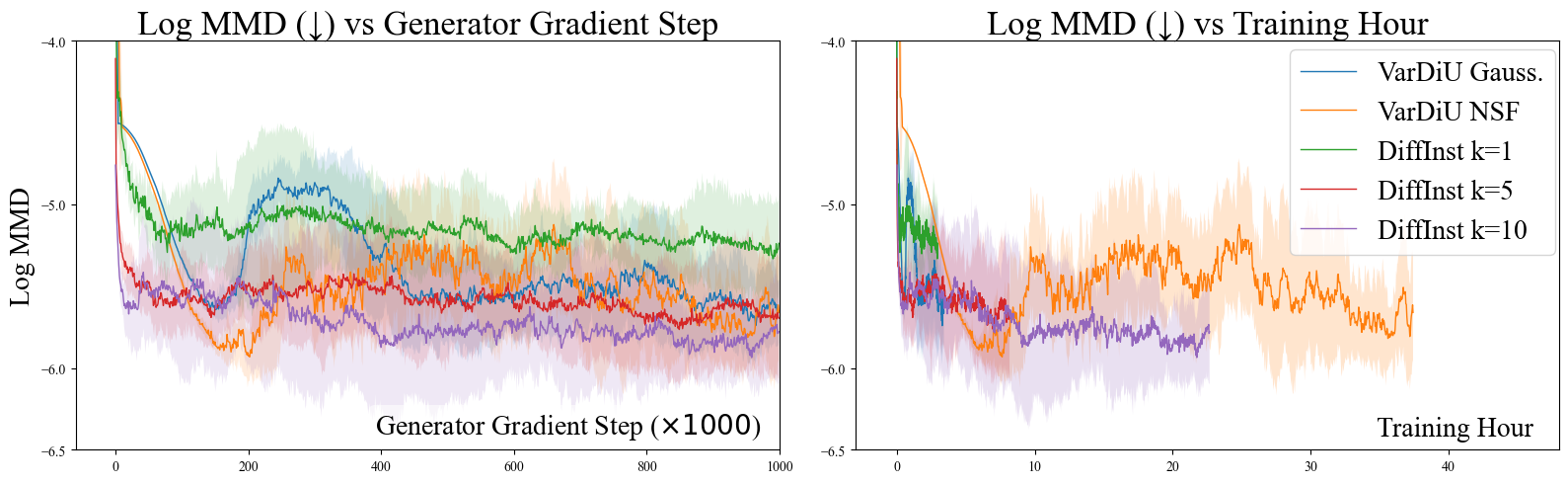}
        \caption{Log-MMD trajectory with pre-trained DM.}
    \end{subfigure}
    \begin{subfigure}[b]{0.98\textwidth}
        \centering
        \includegraphics[width=\linewidth]{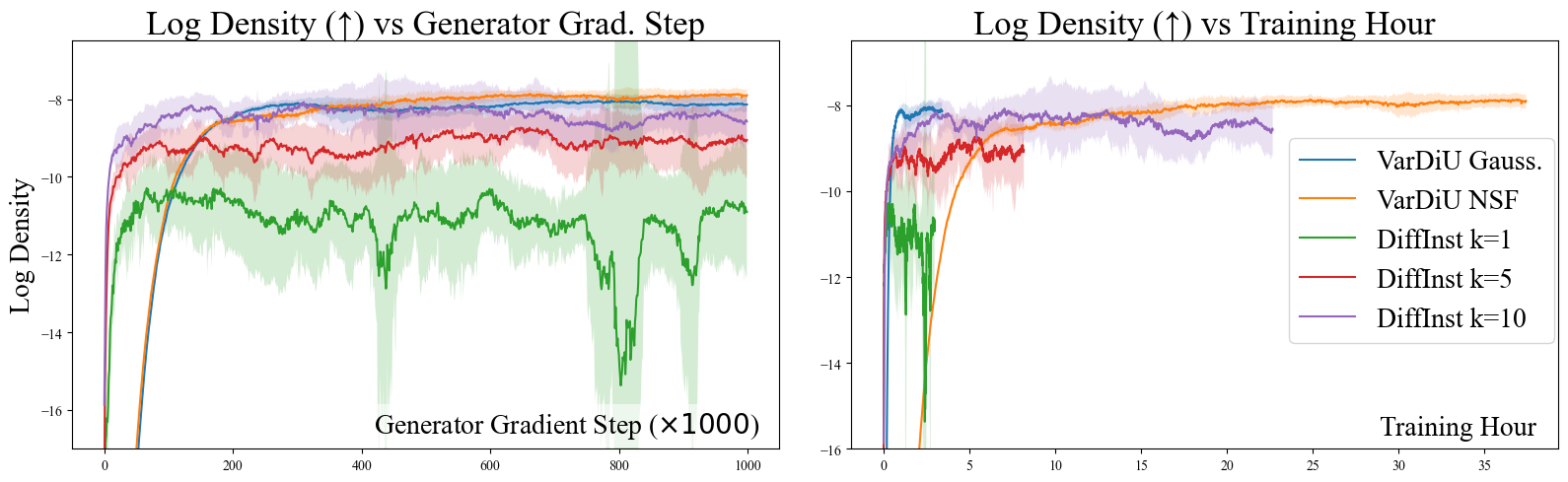}
        \caption{Log-density trajectory with pre-trained DM..}
    \end{subfigure}
    \caption{Comparison of samples log-density and log-MMD trajectories with learned score.}
    \label{fig:speed_fake}
\end{figure}

\begin{figure}[h]
    \centering
    \begin{subfigure}[b]{0.98\textwidth}
        \centering
        \includegraphics[width=\linewidth]{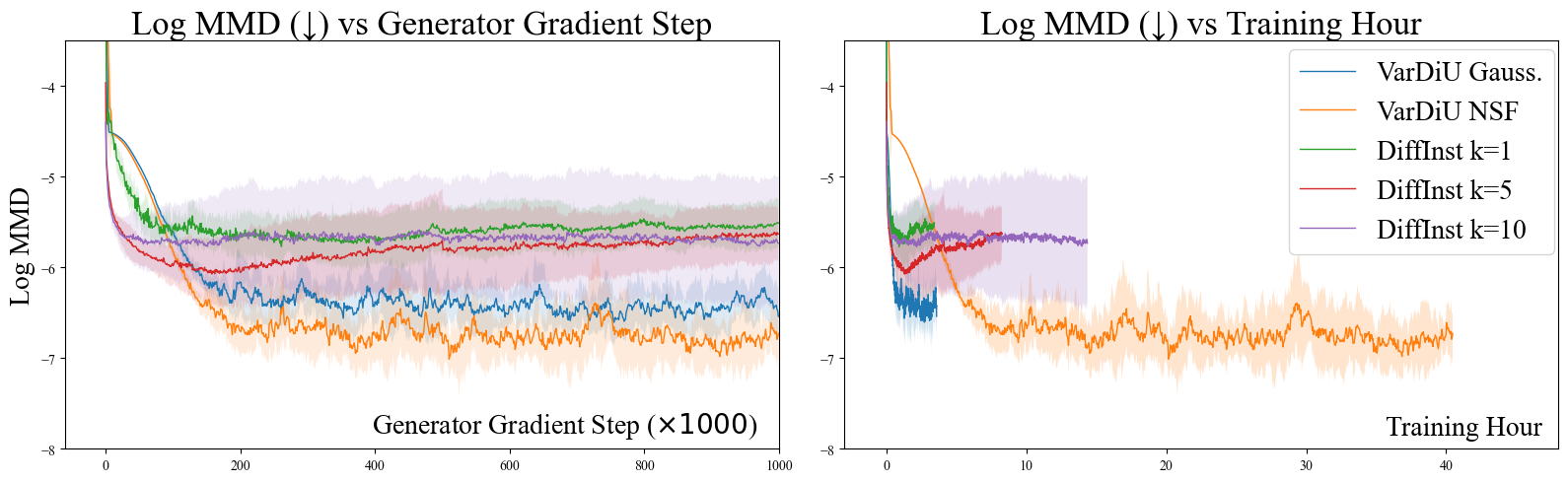}
        \caption{Log-MMD trajectory with a given dataset.}
    \end{subfigure}
    \begin{subfigure}[b]{0.98\textwidth}
        \centering
        \includegraphics[width=\linewidth]{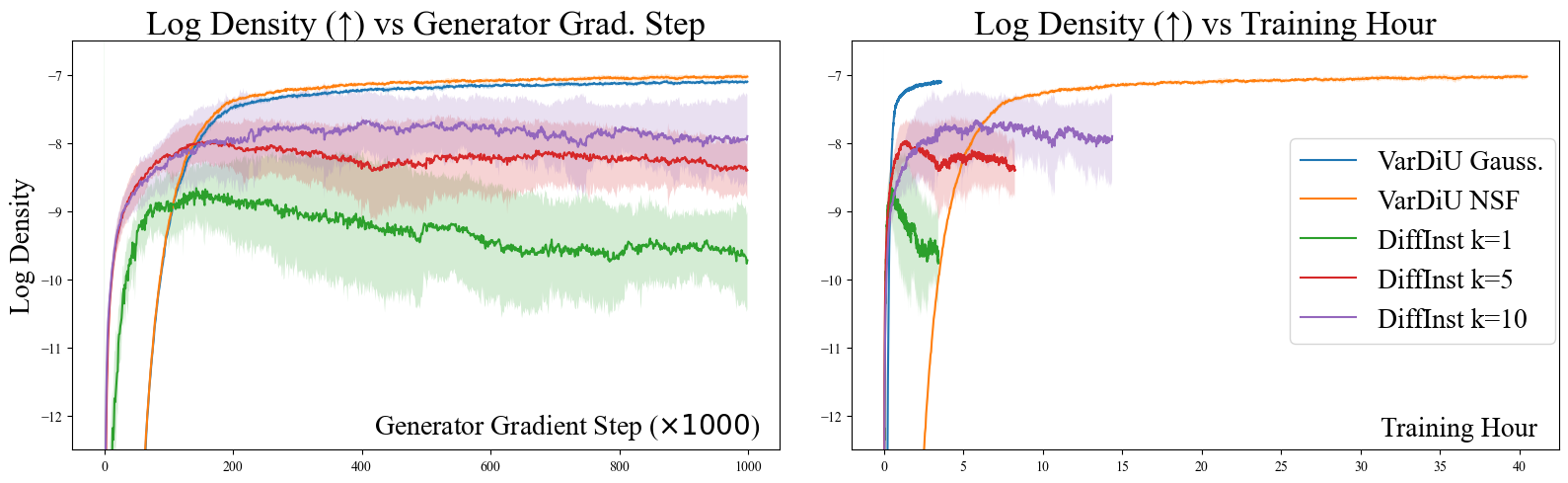}
        \caption{Log-density trajectory with a given dataset.}
    \end{subfigure}
    \caption{Comparison of samples log-density and log-MMD trajectories under setting of given dataset.}
    \label{fig:speed_emp}
\end{figure}

\end{document}